\documentclass{article}

\usepackage{microtype}
\usepackage{graphicx}
\usepackage{subfigure}
\usepackage{booktabs} % for professional tables
\usepackage{yhmath}

\usepackage{hyperref}

\usepackage{algorithm}
\usepackage[algo2e,noend,ruled,linesnumbered]{algorithm2e} 
\usepackage{amsmath}
\usepackage{amssymb}
\usepackage{mathtools}
\usepackage{amsthm}
\usepackage{dsfont}

\usepackage[capitalize,noabbrev]{cleveref}
\theoremstyle{plain}
\newtheorem{theorem}{Theorem}[section]
\newtheorem{proposition}[theorem]{Proposition}
\newtheorem{lemma}[theorem]{Lemma}
\newtheorem{corollary}[theorem]{Corollary}
\theoremstyle{definition}

\newtheorem{observation}[theorem]{Observation}
\theoremstyle{remark}

\usepackage[textsize=tiny]{todonotes}

\DeclareMathOperator*{\argmax}{arg\,max}

\newcommand{\dimn}{\textsf{dim}}
\newcommand{\thr}{\textsf{thr}}
\newcommand{\Thr}{\textsf{Thr}}
\newcommand{\indic}[1]{\mathds{1}_{#1}}

\newcommand{\lpos}{\ensuremath{\textsf{blue}}}
\newcommand{\lneg}{\ensuremath{\textsf{red}}}

\newcommand{\ethlong}{Exponential Time Hypothesis}

\newcommand{\val}{\ensuremath{\textsf{val}}}
\newcommand{\enf}{\ensuremath{\textsf{enf}}}
\newcommand{\lef}{\ensuremath{\textsf{left}}}
\newcommand{\rig}{\ensuremath{\textsf{right}}}

\newcommand{\cla}{\ensuremath{\textsf{cla}}}

\newcommand{\Oh}{\mathcal{O}}
\newcommand{\dem}{\ensuremath{\textsf{dem}}}

\newcommand{\poly}{\ensuremath{\operatorname{poly}}}

\DeclarePairedDelimiter{\abs}{\lvert}{\rvert}

\newcommand{\pMTESlong}{\textsc{Minimum Tree Ensemble Size}}
\newcommand{\pMMTESlong}{\textsc{Minimax Tree Ensemble Size}}
\newcommand{\pMTES}{\textsc{MTES}}
\newcommand{\pMMTES}{\textsc{MmaxTES}}
\newcommand{\pDTSlong}{\textsc{Minimum Decision Tree Size}}
\newcommand{\pDTS}{\textsc{DTS}}

\newcommand{\pMTESOlong}{\textsc{Minimum (Error, Size) Tree Ensemble}}
\newcommand{\pMMTESOlong}{\textsc{Minimax (Error, Size) Tree Ensemble}}
\newcommand{\pMTESO}{\textsc{MESTE}}

\newcommand{\pDTSOlong}{\textsc{Minimum (Error, Size) Tree}}
\newcommand{\pDTSO}{\textsc{MEST}}

\newcommand{\pHS}{\textsc{Hitting Set}}
\newcommand{\pDS}{\textsc{Dominating Set}}
\newcommand{\abColor}{\textsc{Multicoloring}}
\usepackage{scrextend}          %Labeling environment

\newcommand{\prob}[6]{%
  \begin{quote}
  	\begin{samepage}
    \begin{labeling}{#6}%
      \setlength\topsep{-.6ex} \setlength\itemsep{-.8ex}
    \item[#1]
    \item[\emph{#2}]#3
    \item[\emph{#4}]#5
    \end{labeling}%
	\end{samepage}
  \end{quote}%
}

\usetikzlibrary{patterns}

\newcommand{\probdef}[3]{\prob{#1}{Instance:}{#2}{Question:}{#3}{as}}

\usepackage{etoolbox}
\newcommand{\appsymb}{$\bigstar$}
\newcommand{\appref}[1]{\hyperref[proof:#1]{\appsymb}}

\begin{document}

\title{On Computing Optimal Tree Ensembles\footnote{This work was initiated at the research retreat of the Algorithmics and Computational Complexity group of TU Berlin held in Darlingerode in September 2022.

An extended abstract of this work appeared in the Proceedings of the 40th International Conference on Machine Learning (ICML ’23) held in Honolulu, USA.
  This full version contains all missing proofs, some of which have been rewritten to make them more accessible, improved running times for computing minimum-size decision trees, an ETH-based lower bound for \textsc{Set Multicover}, and a new section on adaptions of our approaches to more classes, minimizing misclassifications, and enumeration.}}

\author{Christian Komusiewicz$^1$ %\orcidlink{0000-0002-6789-2918}
 \and Pascal Kunz$^{2}$ \footnote{Supported by the DFG Research Training Group 2434 ``Facets of Complexity''.} %\orcidlink{0000-0003-0829-7032} 
 \and Frank Sommer$^3$ \footnote{Supported by the DFG, project EAGR (KO~3669/6-1)  and by the Alexander von Humboldt Foundation.} \and Manuel Sorge$^3$ \footnote{Supported by the Alexander von Humboldt Foundation.} %\orcidlink{0000-0003-4034-525X}
 }
\date{%
     $^1$ Fakultät für Mathematik und Informatik, Friedrich-Schiller-Universität Jena, Germany\\%
     $^2$ Algorithm Engineering, HU Berlin, Germany\\
     $^3$ Institute of Logic and Computation, TU Wien, Austria
}

\maketitle

\begin{abstract}%   <- trailing '%' for backward compatibility of .sty file
  Random forests and, more generally, (decision\nobreakdash-)tree ensembles are widely used methods for classification and regression.
Recent algorithmic advances allow to compute decision trees that are optimal for various measures such as their size or depth.
We are not aware of such research for tree ensembles and aim to contribute to this area.
Mainly, we provide two novel algorithms and corresponding lower bounds.
First, we are able to carry over and substantially improve on tractability results for decision trees:
We obtain an algorithm that, given a training-data set and an size bound $S \in \mathbb{R}$, computes a tree ensemble of size at most $S$ that classifies the data correctly.
The algorithm runs in $(4\delta D S)^S \cdot \poly$-time, where $D$ the largest domain size, $\delta$ is the largest number of features in which two examples differ, $n$ the number of input examples, and $\poly$ a polynomial of the input size.
For decision trees, that is, ensembles of size~1, we obtain a running time of~$(\delta D s)^s \cdot \poly$, where~$s$ is the size of the tree.
To obtain these algorithms, we introduce the witness-tree technique, which seems promising for practical implementations.
Secondly, we show that dynamic programming, which has been applied successfully to computing decision trees, may also be viable for tree ensembles, providing an $\ell^n \cdot \poly$-time algorithm, where $\ell$ is the number of trees.
% We also provide matching lower bounds.
Finally, we compare the number of cuts necessary to classify training data sets for decision trees and tree ensembles, showing that ensembles may need exponentially fewer cuts for increasing number of trees.
\end{abstract}

%\begin{keywords}
%  Decision Trees, Parameterized complexity, Exponential Time Hypothesis, NP-hard problem, Set Cover Conjecture
%\end{keywords}

\section{Introduction}

Random forests are a method for classification or regression in which we construct an ensemble of decision trees for (a random subsets of) the training data and, in the classification phase, aggregate their outcomes by majority voting.
The random-forests method has received a tremendous amount of attention for its simplicity and improved accuracy over plain decision trees~\cite{breiman_random_2001,verikas_mining_2011,kulkarni_random_2013,rokach_decision_2016}.
Commonly, % empirically\todo{Remove ``empirically''? Most of these heuristic are also easily seen to be fast in theory?}
fast heuristics without performance guarantees are used for computing random forests~\cite{kulkarni_random_2013,rokach_decision_2016}, in particular for computing the individual decision trees in the forest.
For plain decision trees, researchers lately made several advances in computing optimal decision trees, that is, decision trees that provably optimize criteria such as minimizing the tree size~\cite{BessiereHO09,NijssenF10,narodytska_learning_2018,bessiere_minimising_2009,VerwerZ19,JanotaM20,AglinNS20,avellaneda_efficient_2020,SCM23,carrizosa_mathematical_2021,demirovic_murtree_2022,CostaP23,SchidlerS24}.
With that increased amount of attention also came theoretical advances, showing the limits and opportunities for developing efficient exact algorithms for computing decision trees~\cite{OrdyniakS21,kobourov_influence_2022,EibenOrdyniakPaesaniSzeider23,GZ24}.
One impetus to computing optimal decision trees is that it is thought that minimizing the size reduces tendencies to overfitting~\cite{bessiere_minimising_2009,demirovic_murtree_2022}.
It is conceivable that such benefits transfer to globally optimizing the tree ensembles computed by random forests. 
However, apart from sporadic hardness results~\cite{tamon_boosting_2000}, we are not aware of exact algorithmic research for tree ensembles. 
In this work, we aim to initiate this direction; that is, we begin to build the theoretical footing for computing optimal tree ensembles and provide potential avenues for exact algorithms that are guaranteed to compute optimal results with acceptable worst-case running times.%\todo{tractable running times $\rightarrow$ acceptable worst case running times? (tractable is the problem not the running time? also later ``tractable algorithm'' sounds off)}

\looseness=-1
We study the algorithmic properties of two canonical formulations of the training problem for tree ensembles:
We are given a set of training examples labeled with two classes and a number $\ell$ of trees and we want to compute a tree ensemble containing $\ell$ trees that classifies the examples consistently with the given class labels.\footnote{To keep the presentation focused, we consider mainly the case without training error. See \cref{sec:extensions} for extensions of some of our results to minimizing training error.}
We want to minimize either the sum of the tree sizes, resulting in the problem \pMTESlong~(\pMTES), or the largest size of a tree in the ensemble, resulting in the problem \pMMTESlong~(\pMMTES).\footnote{It is also natural to consider the depths of the trees instead of their sizes, but results are usually transferable between these two optimization goals and the size makes the presentation more accessible.}
Both contain as a special case the problem of computing a minimum-size decision tree, which is NP-hard~\cite{hyafil_constructing_1976,OrdyniakS21}.
However, the hardness constructions do not necessarily reflect practical data.
Thus, we are interested in precisely which properties make the problems hard or tractable.
% We measure these properties as integer-valued parameters $p$ and aim to find algorithms with running time\footnote{Here $f$ is an ideally slow-growing arbitrary function and $\poly$ a polynomial of ideally low degree in the input size.} $f(p) \cdot \poly$ or prove their likely absence (so-called W[$t$]-hardness, $t \geq 1$).

Mainly, we provide two novel algorithms for \pMTES\ and \pMMTES\footnote{The algorithms work on the decision version of these problems, but they easily apply to the optimization versions as well.} and matching lower-bound results for their running times.
We call the first one \emph{witness-tree algorithm}. This algorithm demonstrates that prospects for tractable algorithms for optimizing decision trees can be non-trivially generalized to optimizing tree ensembles.
Namely, it was known that for small tree size $s$, moderate maximum domain size $D$ of any feature, and moderate number $\delta$ of features in which two examples differ\footnote{See \cite{OrdyniakS21} for measurements showing that this is a reasonably small parameter in several datasets.}, a minimum decision tree can be computed efficiently, that is, in $f(s, D, \delta) \cdot \poly$ time, where $\poly$ is a polynomial in the input size~\cite{OrdyniakS21}.
However, the function $f$ is at least $\delta^s \cdot (D^s2\delta)^s \cdot 2^{s^2}$ and the algorithm involves enumerative steps in which the worst-case running time equals the average case.
We show that, even for the more general \pMTES, we can improve the running time to $\Oh((4\delta DS)^S \cdot S \ell d n)$, where $S$ denotes the sum of the tree sizes, $\ell$ the number of trees in the ensemble, $n$ the number of training examples, and $d$ the number of dimensions or features of the input data (\cref{thm:witness-tree-algo}).
Moreover, we can avoid the enumerative approach, obtaining a search-tree algorithm that is both conceptually simpler and more easily amenable to heuristic improvements such as early search-termination rules and data reduction.
We achieve this by growing the trees iteratively and labeling their leaves with witness examples that need to be classified in these leaves.
This allows us to localize misclassifications and their rectification, shrinking the search space.
We believe that this technique may have practical applications beyond improving the worst-case running times as we do here.
The running time that we achieve is tight in the sense that we cannot decrease the exponent to $o(S)$ without violating reasonable complexity-theoretic assumptions~(\cref{thm:witness-tree-algo-tight}).

In the time since the preliminary version of this article appeared, \cite{EibenOrdyniakPaesaniSzeider23} showed that, for computing minimum-size decision trees, the dependency of the exponential part of the running time on the domain size~$D$ can be dropped, that is, there is an $f(s, \delta) \cdot \poly$-time algorithm.
It still uses enumerative steps which would be infeasible in practice and which the witness-tree approach avoids.
Furthermore, in the meantime our approach has been shown to be of more general relevance, as it applies not only to decision trees and tree ensembles, but also decision sets, decision lists, and their ensembles~\cite{OPRS24}.

\looseness=-1
As to our second main contribution, recently, exponential-time dynamic programming has been applied to compute optimal decision trees and the resulting trees have shown comparable performance to (heuristic) random forests on some datasets~\cite{demirovic_murtree_2022}.
With the second algorithm that we provide, we investigate the potential of dynamic programming for computing optimal tree ensembles.
We first show that minimizing decision trees can be done in $\Oh(2^n \cdot Ddn)$~time by a dynamic-programming approach that works on all possible splits of the examples (\cref{cor:single-exponential-decision-trees}).\footnote{Indeed, the algorithm employed by \cite{demirovic_murtree_2022} similarly computes a table over all possible splits in the worst case.}
We then extend this algorithm to tree ensembles with $\ell$ trees, achieving $\Oh((\ell + 1)^n \cdot Ddn)$ running time~(\cref{thm:exptime-algo}).
Unfortunately, we also show that these running times cannot be substantially improved:
A running time of $(2 - \epsilon)^n$ for any $\epsilon > 0$ or $f(\ell) \cdot 2^{o(\log \ell)\cdot n}\cdot\poly$, would violate reasonable complexity-theoretic assumptions~(\cref{thm:witness-tree-algo-tight,thm:superexponential-lower-bound}).

Finally, we compare the power of decision trees and tree ensembles in terms of their sizes.
Here, we show that a training data set $\mathcal{D}$ that can be classified by a tree ensemble with $\ell$ trees of size at most $s$, can also be classified by a decision tree of size $(s + 1)^{\ell}$~(\cref{thm:ensemble-to-tree-upper-bound}).
However, such an exponential increase is necessary in the worst case: We show that there exist training data sets $\mathcal{D}$ that cannot be classified by any decision tree of size roughly $(s/2)^{\ell/2}$~(\cref{thm:ensemble-to-tree-lower-bound}).

In the above we focus on binary classification without misclassifications. In \cref{sec:extensions} we explain how our results extend to more classes and the situation where misclassifications are allowed. We suggest directions for future research in \cref{sec:outlook}.

In summary, as the number of trees in a tree ensemble grow, the classification power increases exponentially over decision trees.
Nevertheless, we are able to carry over and improve on tractability results for decision trees if in particular the number of cuts in the optimal ensemble is relatively small.
The underlying witness-tree technique seems promising to try in practice.
Furthermore, we show that dynamic programming, which has been successful for decision trees, may also be viable for tree ensembles.
We provide matching running time lower bounds for all of our algorithms.
Apart from tuning our algorithms, in the future, deconstructing these lower bounds, that is, comparing the constructed instances with the real world~\cite{komusiewicz_deconstructing_2011}, may provide further guidelines towards which properties of the input data we may exploit for efficient algorithms and which we likely may not.

\section{Preliminaries}

\label{sec:prelims}
% and formal problem definition}

For $n \in \mathbb{N}$ we use $[n] \coloneqq \{1, 2, \ldots, n\}$.
For a vector $x \in \mathbb{R}^d$ we denote by $x[i]$ the $i$th entry of~$x$.
% Use bracket notation to enable using bottom indices for points/examples.

Let $\Sigma$ be a set of class symbols; unless stated otherwise, we use $\Sigma = \{\lpos, \lneg\}$.
A decision tree in $\mathbb{R}^d$ with classes $\Sigma$ is formally defined as follows.
Let $T$ be an ordered binary tree, that is, each inner node has a well-defined left and right child.
Let $\dimn \colon V(T) \to [d]$ and $\thr \colon V(T) \to \mathbb{R}$ be labelings of each internal node~$q \in V(T)$ by a \emph{dimension} $\dimn(q) \in [d]$ and a \emph{threshold} $\thr(q) \in \mathbb{R}$.
Furthermore, let $\cla(\ell) \colon V(T) \to \Sigma$ be a labeling of the leaves of~$T$ by class symbols.
Then the tuple $(T, \dimn, \thr, \cla)$ is a \emph{decision tree} in $\mathbb{R}^d$ with classes~$\Sigma$.
We often omit the labelings $\dimn, \thr, \cla$ and just refer to the tree $T$.
The \emph{size} of $T$ is the number of its internal nodes.
We call the internal nodes of $T$ and their associated labels \emph{cuts}.

A \emph{training data set} is a tuple $(E, \lambda)$ of a set of \emph{examples} $E \subseteq \mathbb{R}^d$ and their class labeling $\lambda \colon E \to \Sigma$.
Given a training data set, we fix for each dimension $i$ a minimum-size set~$\Thr(i)$ of thresholds that distinguishes between all values of the examples in the~$i$th dimension.
In other words, for each pair of elements~$e$ and~$e'$ with~$e[i] < e'[i]$, there is at least one value~$h\in \Thr(i)$ such that~$e[i] < h < e'[i]$.
Let~$t \in \mathbb{R}$ be some threshold.
We use~$E[f_i \leq h] = \{ x\in E \mid x[i] \leq h\}$ and~$E[f_i > h] = \{ x\in E \mid x[i] > h\}$ to denote the set of examples of~$E$ whose $i$th dimension is less or equal and strictly greater than the threshold~$h$, respectively.

Now let $T$ be a decision tree.
Each node $q \in V(T)$, including the leaves, defines a subset $E[T, q] \subseteq E$ as follows.
For the root~$r$ of~$T$, we define~$E[T,r] \coloneqq E$.
For each non-root node~$q\in V(T)$, let~$p$ denote the parent of~$q$.
We then define~$E[T,q] \coloneqq E[T,p] \cap E[f_{\dimn(p)} \le \thr(p)]$ if $q$ is the left child of~$p$ and~$E[T,q] \coloneqq E[T,p] \cap E[f_{\dimn(p)} > \thr(p)]$ if $q$ is the right child of~$p$.
If the tree $T$ is clear from the context, we simplify $E[T, q]$ to $E[q]$.
We say that~$T$ \emph{classifies}~$(E, \lambda)$ if for each leaf $u \in V(T)$ and each example $e \in E[u]$ we have~$\lambda(e) = \cla(u)$ (recall that $\cla$ is a labeling of the leaves of~$T$ by classes).
Note that the set family that contains $E[u]$ for all leaves $u$ of $T$ forms a partition of $E$.
Thus for each example $e \in E$ there is a unique leaf $u$ such that $e \in E[u]$.
We also say that $u$ is \emph{$e$'s leaf}.
We say that $\cla(u)$ is the class \emph{assigned} to $e$ by $T$ and we write $T[e]$ for $\cla(u)$.

A \emph{tree ensemble} is a set of decision trees.
A tree ensemble~$\mathcal{T}$ \emph{classifies} $(E, \lambda)$ if for each example $e \in E$ the majority vote of the trees in $\mathcal{T}$ agrees with the label $\lambda(e)$.
That is, for each example $e \in E$ we have $\lambda(e) = \argmax_{\sigma \in \Sigma} |\{ T \in \mathcal{T} \mid T[e] = \sigma \}| \eqqcolon \mathcal{T}[e]$.
To avoid ambiguity in the maximum, we fix an ordering of $\Sigma$ and break ties according to this ordering.
If $\Sigma = \{\lpos, \lneg\}$ we break ties in favor of $\lpos$.
The \emph{overall size} of a tree ensemble $\mathcal{T}$ is the sum of the sizes of the decision trees in $\mathcal{T}$.

We consider the following computational problem.

\probdef{\pMTESlong\ (\pMTES)}
{A training data set $(E, \lambda)$, a number~$\ell$ of trees, and a size bound~$S$.}
{Is there a tree ensemble of overall size at most $S$ that classifies $(E, \lambda)$ and contains exactly~$\ell$ trees?}

When restricted to $\ell = 1$, \pMTES\ is known as \pDTSlong\ (\pDTS)~\cite{OrdyniakS21,kobourov_influence_2022}.
In the problem variant \pMMTESlong\ (\pMMTES), instead of $S$, we are given an integer $s$ and ask whether there is a tree ensemble that classifies $(E, \lambda)$ and contains exactly~$\ell$ trees, each of which has size at most~$s$.

Our analysis is within the framework of parameterized complexity~\cite{gottlob_fixedparameter_2002,flum_parameterized_2006,Nie06,CyFoKoLoMaPiPiSa2015,downey_fundamentals_2013}.
Let $L \subseteq \Sigma^{*}$ be a computational problem specified over some alphabet~$\Sigma$ and let $p \colon \Sigma^{*} \to \mathbb{N}$ be a parameter, that is,~$p$ assigns to each instance of $L$ an integer parameter value (which we simply denote by $p$ if the instance is clear from the context).
We say that $L$ is \emph{fixed-parameter tractable}~(FPT) with respect to $p$ if $L$ can be decided in $f(p) \cdot \poly(n)$~time where $n$ is the input encoding length.
The corresponding hardness concept related to fixed-parameter tractability is W[$t$]-hardness,~$t \geq 1$; if the problem $L$ is W[$t$]-hard with respect to $p$, then $L$ is thought to not be fixed-parameter tractable; see \cite{flum_parameterized_2006,Nie06,CyFoKoLoMaPiPiSa2015,downey_fundamentals_2013} for details.
The Exponential Time Hypothesis (ETH)~\cite{impagliazzo_complexity_2001,impagliazzo_which_2001} states that \textsc{3SAT} on $n$-variable formulas cannot be solved in $2^{o(n)}$~time.
The Set Cover Conjecture~\cite{CyganDLMNOPSW16} states that \textsc{Set Cover} cannot be solved in $(2-\varepsilon)^n$~time for any~$\varepsilon>0$.

% \section*{Acknowledgements}

% Important ideas for this article were developed in the % relaxed % atmosphere of the 2022 retreat of the AKT group of TU Berlin.
% Manuel Sorge acknowledges funding by the Alexander von Humboldt Foundation.

% \textbf{Do not} include acknowledgements in the initial version of
% the paper submitted for blind review.

% If a paper is accepted, the final camera-ready version can (and
% probably should) include acknowledgements. In this case, please
% place such acknowledgements in an unnumbered section at the
% end of the paper. Typically, this will include thanks to reviewers
% who gave useful comments, to colleagues who contributed to the ideas,
% and to funding agencies and corporate sponsors that provided financial

% List of parameters
% \begin{itemize}
% \item $s$ size of tree
% \item $d$ number of features/dimensions
% \item $D$ domain size
% \item $\delta$ max \# dimensions in which red, blue example differs
% \item $\ell$ number of trees in ensemble
% \item \todo[inline]{Use, e.g.~$S$ for the total number of cuts in tree ensemble?}
% \end{itemize}

\section{Decision Trees Versus Tree Ensembles}
\label{sec:treesVensembles}

We will call a decision tree and a tree ensemble \emph{equivalent} if any training data set is classified by the one if and only if it is classified by the other.
We start by comparing the minimum size of a decision tree to that of a minimum-size decision tree ensemble, showing that there are examples where the latter is significantly smaller.
\cite{VidalS20} obtained similar results for the depth of decision trees and decision tree ensembles, showing that any training data set that can be classified by a decision tree ensemble with $\ell$ trees of depth~$d$ can also be classified by a tree with depth $\ell\cdot d$ and that this bound is tight.
Here, we analyze trees and tree ensembles in terms of their size, showing that a minimum ensemble can be exponentially smaller than any equivalent decision tree.

\begin{observation}
	If $s$ is the size of a decision tree and $L$ the number of leaves, then $s = L -1$.
\end{observation}

\begin{theorem}
	\label{thm:ensemble-to-tree-upper-bound}
	Any training data set that can be classified by a decision tree ensemble consisting of $\ell$ trees, each of size at most $s$, can also be classified by a decision tree of size $(s+1)^\ell -1$.
\end{theorem}

\begin{proof}
  \cite{VidalS20} showed that given a tree ensemble $\mathcal T = \{T_1,\ldots,T_\ell\}$, such that each $T_i$ has size at most $s$, there is an equivalent decision tree $T$ obtained by appending $T_{i+1}$ to every leaf of a tree that is equivalent to~$\mathcal T_i \coloneqq \{T_1,\ldots,T_i\}$ (and taking the class labels of the leaves to be the majority vote of the trees in $\mathcal{T_i}$).
  If $s_i$ is the size of the tree equivalent to $\mathcal T_i$ obtained in this manner, then this tree contains $s_i + 1$ leaves.
  Hence, $s_{i+1} \leq s_i + (s_i + 1) s$.
  By a simple induction, this implies that $s_i \le (s+1)^i -1$ and, therefore, the size of $T$ is $s_\ell \le (s+1)^\ell -1$.
\end{proof}

Next, we show that an exponential blow-up in $\ell$ is necessary.

\begin{theorem}
	\label{thm:ensemble-to-tree-lower-bound}
	For any odd $\ell,s \in \mathbb N$, there is a training data set that can be classified by a decision tree ensemble containing $\ell$ trees of size $s$ each, but cannot be classified by a single decision tree of size smaller than \[\frac{(s+1)^\ell}{\ell (\frac{s+1}{2})^{\frac{\ell-1}{2}}}-1.\]
\end{theorem}
%\appendixproof{thm:ensemble-to-tree-lower-bound}{
\begin{proof}
	\newcommand{\ev}{\mathrm{even}}
	\newcommand{\od}{\mathrm{odd}}
	\newcommand{\EV}{\mathrm{ev}}
	\newcommand{\OD}{\mathrm{od}}
	For any $x \in \mathbb N^\ell$, let $\ev(x) \coloneqq \{ i \in [\ell] \mid x[i] \text{ is even}\}$ and $\od(x) \coloneqq [\ell] \setminus \ev(x)$.
	Furthermore, let~$\EV(x)$ and~$\OD(x)$ denote the sizes of~$\ev(x)$ and~$\od(x)$, respectively. 
	Let $E \coloneqq \{x \in [s+1]^\ell \mid \abs{\EV(x) - \OD(x)} = 1\}$ and $\lambda \colon E \to \{\lpos,\lneg\}$ with $\lambda(x) = \lpos$ if and only if~$\EV(x) > \OD(x)$.
	%We claim that $(E,\lambda)$ can be classified by a decision tree ensemble consisting of $\ell$ trees each containing $s$ nodes and that any decision tree that classifies $(E,\lambda)$ contains at least $\binom{\ell}{\frac{\ell+1}{2}}(\frac{s}{2})^{\frac{\ell+1}{2}}$ leaves.
	We show that $(E,\lambda)$ fulfills the claim.

	First, we will show that there is a decision tree ensemble $\mathcal T$ with $\ell$~trees of size at most~$s$ that classifies $(E,\lambda)$.
	For each dimension $i \in [\ell]$, $\mathcal T$ contains a tree $T_i$ with $s$ inner nodes that checks for a given example~$x$ whether $x[i]$ is even or odd and reaches a $\lpos$ or $\lneg$ leaf, respectively.
	Such a decision tree is pictured in \cref{fig:odd-even-tree}.
	Then, a majority of the trees classifies a given example as $\lpos$ if and only if it contains more even than odd entries, which is equivalent to its label being $\lpos$.
	
	\begin{figure}
		\centering
		\begin{tikzpicture}[xscale=1.25,yscale=0.5]
		\node[circle,draw] (v1) at (1,10) {$ \scriptstyle x[i] \leq 1$};
		\node[circle,fill,red] (l1) at (0,9) {};
		\draw[->] (v1) -- (l1);
		\node[circle,draw] (v2) at (2,9) {$ \scriptstyle x[i] \leq 2$};
		\draw[->] (v1) -- (v2);
		\node[circle,fill,blue] (l2) at (1,8) {};
		\draw[->] (v2) -- (l2);
		\node[circle,draw] (vs-1) at (4,7) {$ \scriptstyle x[i] \leq s$};
		\draw[->] (v2) --  node[pos=0.5,fill=white] {$\ldots$}  (vs-1);
		\node[circle,fill,red] (ls-1) at (3,6) {};
		\node[circle,fill,blue] (ls) at (5,6) {};
		\draw[->] (vs-1) -- (ls-1);
		\draw[->] (vs-1) -- (ls);
		\end{tikzpicture}
		\caption{The tree~$T_i$ of~$\mathcal{T}$.}
		\label{fig:odd-even-tree}
	\end{figure}

	We will show that any decision tree $T$ that classifies~$E$ has at least $$\frac{(s+1)^\ell}{\ell (\frac{s+1}{2})^{\frac{\ell-1}{2}}}$$ leaves.
	Observe that $\abs{E} \ge \binom{\ell}{\frac{\ell+1}{2}} (\frac{s+1}{2})^\ell \ge \frac{2^\ell}{\ell} (\frac{s+1}{2})^\ell = \frac{(s+1)^\ell}{\ell} $, because even if we fix some subset $I$ of~$[\ell]$ such that $x[i]$ is to be even if and only if $i\in I$, then there are $\frac{s+1}{2}$ possible values for each component of $x$.
	Therefore, it is sufficient to prove that~$\abs{E[T,q]} \le (\frac{s+1}{2})^{\frac{\ell-1}{2}}$ for every leaf $q$ of~$T$.
	
	Let $q$ be a leaf of $T$.
	Without loss of generality, assume that~$\cla (q) = \lpos$, that is~$\lambda(x) = \lpos$ for all $x\in E[T,q]$.
	We will show that $x[i] = y[i]$ for all $x,y \in E[T,q]$ and~$i \in \ev(x)$.
	Suppose that $x[i] \neq y[i]$, $i \in \ev(x)$, and $x,y \in E[T,q]$.
	Without loss of generality, $x[i] < y[i]$.
	Define~$z\in [s+1]^\ell$ by $z[i] \coloneqq x[i] + 1$ and $z_j \coloneqq x_j$ for all $j \in [\ell] \setminus \{i\}$.
	Then, $\ev(z) = \ev(x) \setminus \{i\}$, implying that $\lambda(z) = \lneg$.
	Hence, $z \notin E[T,q]$.
	This implies that~$T$ contains a node $v$ with $\dim(v) = i$ and $\thr(v) = x[i]$ on the path from the root to~$q$.
	However, this means that $y \notin E[T,q]$.
	Hence, the examples in $E[T,q]$ can differ only in the components in~$\od(x)$.
	Moreover, $\OD(x) = \frac{\ell-1}{2}$ and $[s+1]$ contains $\frac{s+1}{2}$~odd values.
	
	Since the size of a binary tree is the number of leaves minus one, the claim follows.
\end{proof}
%}

This result still leaves a considerable gap between the upper and lower bound. 
We conjecture that the lower bound can be improved: for example, by showing that in the example presented in the proof of \Cref{thm:ensemble-to-tree-lower-bound} the number of examples in each leaf is, on average, smaller than what we showed.

\section{The Witness-Tree Algorithm% for small domains
}\label{sec:witness-tree-algo}

In this section, we prove the first of our two main theorems.
%\todo[inline]{I would not define all quantities in the theorem. They should be explained/defined in the intro as well and here I would recall them above the theorem.} ms: Done.
Recall that~$S$ is the desired overall size of the tree ensemble, $s$ is the maximum size of a tree in the ensemble, $\ell$ is the number of trees in the ensemble, $D$ is the largest domain of a feature, $\delta$ is the largest number of features in which two examples of different classes differ, $d$ the number of features, and $n$ is the number of training examples.
\begin{theorem}\label{thm:witness-tree-algo}
	\pMTESlong\ can be solved in $\Oh((4 \delta D S)^{S} \cdot S \ell d n)$ time and \pMMTESlong\ in $\Oh(2^{\ell} \cdot (\delta D \ell s)^{s\ell} \cdot s \ell^2 d n)$ time.
\end{theorem}

The basic idea is to start with a tree ensemble that contains only trivial trees and to successively refine the trees in the ensemble until all input examples are classified correctly.
To guide the refinement process, each leaf of each tree is assigned a distinct example, called a \emph{witness}.
In a recursive process we then aim to find refinements of the trees that classify more and more examples correctly while maintaining that each witness is classified in the assigned leaf.
Maintaining the witness-leaf assignment will speed up the refinement process because it enables us to detect examples that need to be cut away from witnesses in some trees of the ensemble.

\newcommand{\wit}{\ensuremath{\textsf{wit}}}
Formally, a \emph{witness tree} is a tuple $(T, \dimn, \thr, \cla, \wit)$ wherein $(T, \dimn, \thr, \cla)$ is a decision tree and~$\wit \colon V(T) \to E$ is a mapping from the leaves of $T$ to the set of examples such that for each leaf~$q\in V(T)$ we have~$\wit[q] \in E[q]$.
The images of $\wit$ are called \emph{witnesses}.
Note that a witness is not necessarily classified correctly, that is, we permit~$T[\wit(q)] \neq \lambda(\wit(q))$.
A \emph{witness ensemble} is a set of witness trees.

We aim to successively refine the trees in a witness ensemble until all examples are classified correctly.
For this, an example $e \in E$ is \emph{dirty} for some tree $T$ (or tree ensemble~$\mathcal{F}$) if the label $T(e)$ (or $\mathcal{F}(e)$) assigned to $e$ by $T$ (or by $\mathcal{F}$) is not equal to $\lambda(e)$.
% \todo[inline]{This use of ``resp.'' is not grammatically correct. I would write ``(or~$\mathcal{F}(e)$)'' instead.} ms: I think it's correct (ellipsis) but I don't care so much.

\looseness=-1
Next, we define a refinement of a tree in an ensemble.
All of our refinements will take a dirty example and change the class label assigned to this example by one of the decision trees.
% \todo[inline]{Potentially confusing? We do not change~$\lamba(e)$ but the label assigned by the tree.} %ms: Indeed, changed.
Consider a witness tree $T$ and a dirty example $e$ for $T$.
Intuitively, we take the leaf $q$ of $T$ in which $e$ is classified and consider its witness~$\wit(q)$.
Then we pick a way of introducing into $T$ a new cut on the path from the root to $q$ that cuts apart $\wit(q)$ and $e$.
This then results in a refinement of $T$.

Formally, let $T$ be a witness tree.
A \emph{one-step refinement} $R$ of $T$ is a witness tree constructed in one of the following two ways (illustrated in \Cref{fig:refinement1}):

\begin{figure*}
	\centering
	\begin{tikzpicture}[yscale=0.7,xscale=0.87]
	\node[circle,draw] (r) at (5,10) {$r$};
	\node[circle,draw] (t) at (2,9) {$v$};
	\draw[fill=gray!50] (8,9) node[anchor=north] (T){}
	-- (7,7) node[anchor=north]{}
	-- (9,7) node[anchor=south]{}
	-- cycle;
	\node[] (l) at (8,8) {$T$};
	\node[label=above:{\rotatebox{15}{\scriptsize $x[\dim(r)] \leq \thr(r)$}}] at (3.4,9.1){};			
	\node[label=above:{\rotatebox{-17}{\scriptsize $x[\dim(r)] > \thr(r)$}}] at (6.8,8.8){};	
	
	\draw[-stealth] (r) -- (t) node [midway,above,sloped] {};
	\draw[-stealth] (r) -- (T) node [midway,above,sloped] {};
	
	\begin{scope}[xshift=7.5cm,yscale=0.75,yshift=4cm]
	\draw[fill=gray!50] (5,10) node[anchor=north] (T11){}
	-- (4,8) node[anchor=north](T12){}
	-- (6,8) node[anchor=south] (T13){}
	-- cycle;
	\node[] (l) at (5,9) {$T_1$};
	\node[circle,draw] (u) at (6,7) {$u$};
	
	\draw[fill=gray!50] (7,6) node[anchor=north] (T2){}
	-- (6,4) node[anchor=north](){}
	-- (8,4) node[anchor=south] (){}
	-- cycle;
	\node[] (l) at (7,5) {$T_2$};
	
	\node[circle,draw] (v) at (4,6) {$v$};
	\draw[-stealth] (T13) -- (u);
	\draw[-stealth] (u) -- (v);
	\draw[-stealth] (u) -- (T2);
	\end{scope}

	\end{tikzpicture}
	\caption{Two ways of refining a tree: On the left a new root $r$ and a new leaf~$v$ are introduced. On the right, an existing edge between the subtrees $T_1$ and $T_2$ is subdivided with a vertex $u$ and a new leaf $v$ is introduced.}
	\label{fig:refinement1}
\end{figure*}

\renewcommand{\descriptionlabel}[1]{\hspace{\labelsep}\textit{#1}}
\begin{description}
\item[Root insertion:] Add a new root $r$ to $T$, labeled with a dimension $\dimn(r)$ and threshold~$\thr(r)$, put the old root of~$T$ to be the left or right child of $r$, and put the other child of~$r$ to be a new leaf $v$, labeled with an arbitrary class label and with a witness $x \in E$ such that~$x \in E[R, v]$ (left part of \Cref{fig:refinement1}).

\item [Edge subdivision:] Pick any edge $f$ in $T$.
  Subdivide $f$ with a new node~$u$, labeled with a dimension $\dimn(u)$ and threshold~$\thr(u)$, and add a new leaf $v$ as a child to $u$, labeled with an arbitrary class label and with a witness $x \in E$.
  The order of the children of $u$ is chosen such that~$x \in E[R, v]$ (right part of \Cref{fig:refinement1}).
\end{description}
This finishes the definition of a one-step refinement.
We also say that the one-step refinement \emph{introduces} the new leaf $v$, the witness $x$, and the node $r$ or $u$ (thought of as the nodes including their associated labelings), respectively.
Observe that the refinement is a witness tree and thus the previous witnesses need to be preserved.
That is, the choices of the dimension $\dimn(r), \dimn(u)$ and threshold $\thr(r), \thr(u)$ need to be such that each witness is still classified in its leaf.
In formulas, for each leaf $t$ of $T$ it must hold that $\wit(t) \in E[R, t]$.

A \emph{refinement} $R$ of a witness tree $T$ is obtained by a series~$T = T_1, T_2, \ldots, T_k = R$ of witness trees such that, for each $i \in \{2, \ldots, k\}$, tree $T_i$ is a one-step refinement of $T_{i - 1}$.
A decision tree~$R$ (without witness labeling) is a \emph{refinement} of a witness tree~$T$ if there is a labeling of the leaves of~$R$ by witnesses such that (a)~the~labeling results in a witness tree, that is, for each leaf $q$ of $R$ the witness $\wit(q)$ of $q$ is in $E[R, q]$, and (b) after the labeling the witness tree $R$ it is a refinement of~$T$.
If a tree ensemble~$\mathcal{C}'$ consists of the trees of a witness ensemble~$\mathcal{C}$ or refinements thereof, then we say $\mathcal{C}'$ is a \emph{refinement} of~$\mathcal{C}$.

In the correctness proof for our algorithm we need a property of refinements that shows that the order in which we introduce witnesses is immaterial.

\begin{figure}[t]
  \centering
  \includegraphics{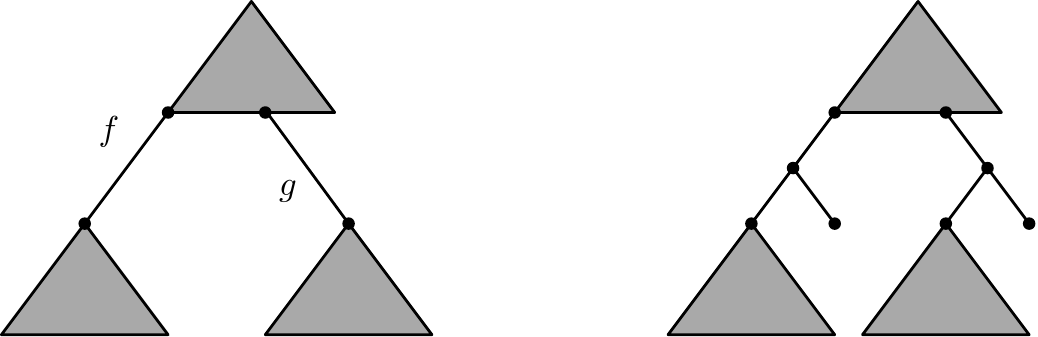}
  \caption{Example for tree $T_1$ (left) and $T_3$ (right) in subcase (1a) in the proof of \cref{lem:reorder-refinements}.}
  \label{reordering-case1a}
\end{figure}
\begin{figure}[t]
  \centering
  \includegraphics{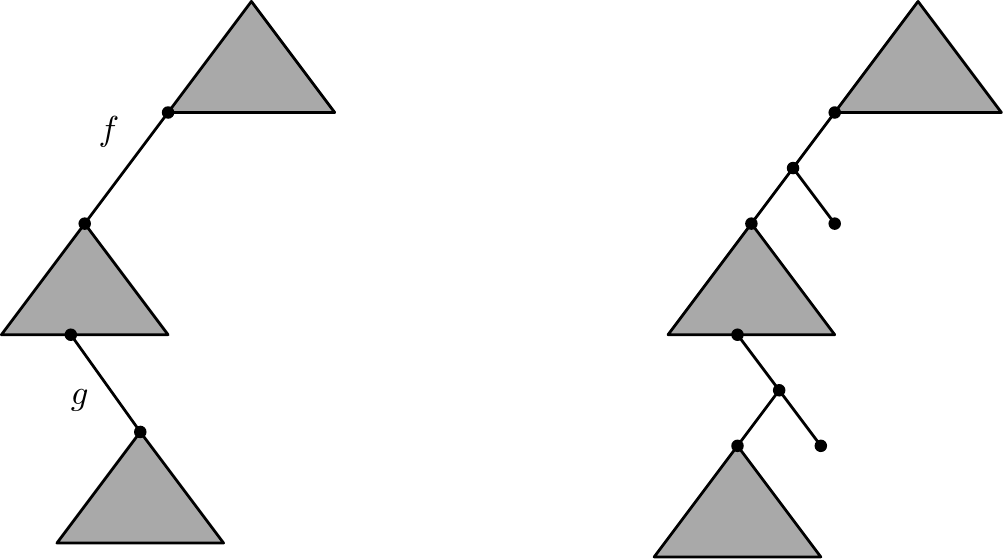}
  \caption{Example for tree $T_1$ (left) and $T_3$ (right) in subcase (1b) in the proof of \cref{lem:reorder-refinements}.}
  \label{reordering-case1b}
\end{figure}
\begin{figure}[t]
  \centering
  \includegraphics{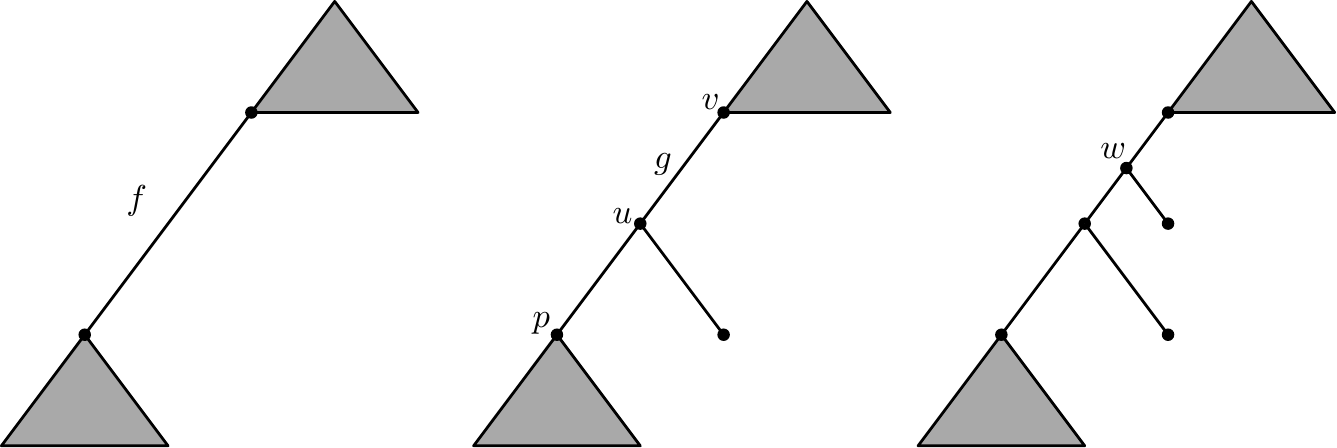}
  \caption{Tree $T_1$ (left), $T_2$ (middle), and $T_3$ (right) in subcase (2a) in the proof of \cref{lem:reorder-refinements}.}
  \label{reordering-case2a}
\end{figure}
\begin{figure}[t]
  \centering
  \includegraphics{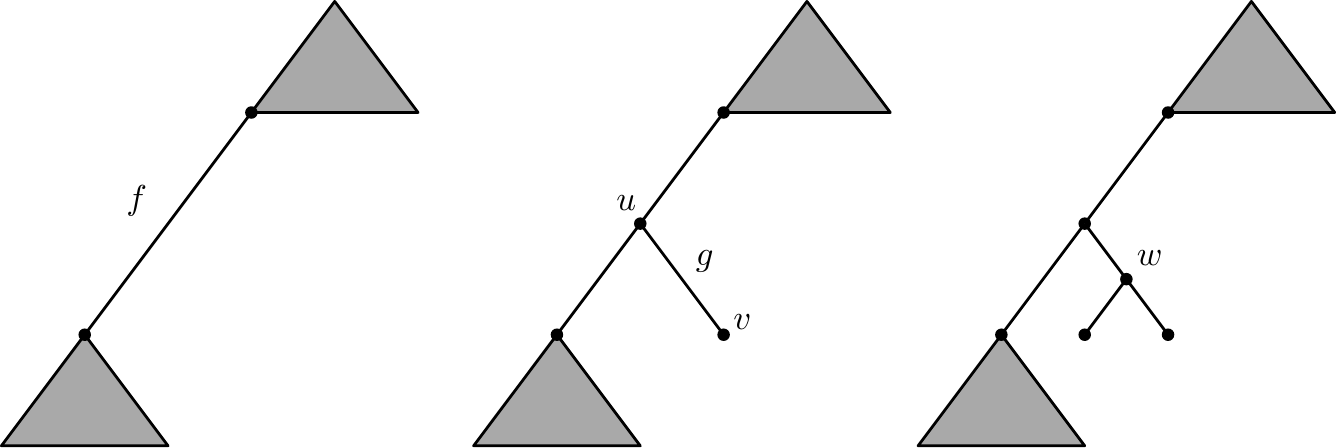}
  \caption{Tree $T_1$ (left), $T_2$ (middle), and $T_3$ (right) in subcase (2b) in the proof of \cref{lem:reorder-refinements}.}
  \label{reordering-case2b}
\end{figure}

\begin{lemma}\label{lem:reorder-refinements}
	Let $T_1$ be a witness tree, $T_2$ a one-step refinement of $T_1$, and $T_3$ a one-step refinement of $T_2$.
	Let $T_3$ introduce a witness $x$.
	Then there is a one-step refinement $S_2$ of $T_1$ that introduces $x$ such that $T_3$ is a one-step refinement of $S_2$.
\end{lemma}
% \appendixproof{lem:reorder-refinements}{
\begin{proof}
  To simplify the proof, observe that introducing a new root in a one-step refinement can be thought of as subdividing an auxiliary edge added to the root of the original tree.
  Thus, below we will only consider one-step refinements that subdivide edges.
  Let $f$ be the edge in $T_1$ that is subdivided to obtain $T_2$ and $g$ the edge in $T_2$ that is subdivided to obtain~$T_3$.
  We distinguish two cases; wether~$g$ is present in~$T_1$ or not:

  In Case (1), $g$ is present in $T_1$, that is, $g$ is not an edge that was introduced in $T_2$ by subdividing $f$.
  There are two subcases:
  
  In Subcase (1a), $f$ and $g$ are not on a common root-leaf path, see \cref{reordering-case1a} for an illustration.
  To obtain $S_2$ from $T_1$ we subdivide $g$ (instead of $f$) in the same way as it is done in $T_2$ to obtain $T_3$.
  Then, $T_3$ is a one-step refinement of $S_2$ that subdivides $f$ (and introduces a witness that was previously introduced in $T_2$).
  Clearly, $S_2$ is a witness tree that introduces $x$, as required.
  
  In Subcase (1b), $f$ and $g$ are on a common root-leaf path, see \cref{reordering-case1b} for an illustration.
  Similarly, we switch the order of subdivision operations.
  To obtain $S_2$, subdivide $g$ in $T_1$ in the same way as done in $T_2$ to obtain $T_3$.
  Then, $T_3$ is a one-step refinement of $S_2$ obtained by subdividing $f$ in the same way as done in $T_1$ to obtain $T_2$.
  Note that $S_2$ is indeed a one-step refinement of $T_1$, that is, no witness of $T_1$ is classified in a different leaf in~$S_2$: Otherwise, such a witness would also be classified in a different leaf in $T_3$, contradicting the fact that $T_2$ and $T_3$ are one-step refinements.
  Finally, $T_3$ is a one-step refinement of $S_2$ for the same reason combined with the fact that subdividing $f$ does not classify the witness~$x$ introduced in $S_2$ in a different leaf.

  In Case (2), $g$ is not present in $T_1$, that is, $T_2$ is obtained by subdividing $f$ with a node~$u$ and $g$ is incident with $u$.
  There are again two subcases:

  In Subcase (2a), the other endpoint, $v$, of $g$ is present in $T_1$, see \cref{reordering-case2a}.
  Let $w$ be the node introduced in $T_3$ by subdividing $g$.
  To obtain $S_2$ from $T_1$ we subdivide $f$ but instead of introducing $u$ we introduce $w$.
  Let $p$ be the other endpoint of $f$ different from $v$.
  Then, $T_3$ is a one-step refinement of $S_2$ obtained by subdividing the edge $\{w, p\}$ with node $u$.
  The argument that $S_2$ and $T_3$ are indeed the required one-step refinements is analogous to Subcase (1b).

  In Subcase (2b), the endpoint $v$ of $g$ that is different from $u$ is not present in $T_1$, see \cref{reordering-case2b}.
  Let $y$ be the witness introduced in $T_2$ and recall that $x$ is the witness introduced in $T_3$.
  Note that both $y$'s and $x$'s leaf in $T_2$ is $v$.
  To obtain $S_2$, we take $T_1$ and subdivide $f$ with the same node $u$ but instead of introducing witness $y$, we introduce witness $x$.
  Note that $g$ is present in $S_2$.
  Let $w$ be the node introduced in $T_3$ by subdividing $g$.
  Then, $T_3$ is a one-step refinement of $S_2$ obtained by subdividing the edge $g$ and introducing $w$ in the same way as in $T_3$.
  Instead of introducing witness $x$, we now introduce witness $y$ instead.
  (Note that, in a one-step refinement, the order of the children of the introduced node is chosen such that the introduced witness is classified correctly in the newly introduced leaf.)

  Concluding, we may replace $T_2$ with $S_2$ as defined above while maintaining the properties of one-step refinements and introducing the witness $x$ in $S_2$ instead.
\end{proof}
% }

\begin{algorithm}[t]
	\DontPrintSemicolon
	\SetKwProg{Fn}{Function}{}{}
	\SetKwFunction{RefineEnsemble}{RefineEnsemble}
	
	\Fn{\RefineEnsemble($\mathcal{C}, (E, \lambda), S$)}{
		
		\KwIn{A witness ensemble $\mathcal{C}$, a training data set~$(E, \lambda)$, and a size threshold $S \in \mathbb{N}$.}
		
		\KwOut{A tree ensemble of overall size at most $S$ that is a refinement of $\mathcal{C}$ and classifies~$(E, \lambda)$ or $\bot$ if none exists.}
		
		\BlankLine
		
		\lIf{overall size of $\mathcal{C}$ is larger than $S$\label{algline:wt-size-check}}{\Return{$\bot$}}
                
		\lIf{$\mathcal{C}$ classifies $(E, \lambda)$}{\Return{$\mathcal{C}$}}

                \lIf{overall size of $\mathcal{C}$ is exactly $S$\label{algline:wt-size-check}}{\Return{$\bot$}}
		
		$e \gets$ a dirty example for $\mathcal{C}$\label{algline:pick-dirty}\;
		
		\For{\label{algline:tree-loop}each tree $T$ in $\mathcal{C}$ in which $e$ is classified incorrectly and not a witness}{
			\For{\label{algline:refinement-loop}each important one-step refinement $T'$ of $T$ introducing $e$ as witness and such that $T'[e] = \lambda(e)$}{
				$\mathcal{C}' \gets$ $\mathcal{C}$ with $T$ replaced by $T'$\;
				$\mathcal{D} \gets $ \RefineEnsemble($\mathcal{C}', (E, \lambda), S$)\;
				\lIf{$\mathcal{D} \neq \bot$}{\Return{$\mathcal{D}$}}
			}
		}
		\Return{$\bot$}
	}  
	\caption{Computing tree ensembles.}
	\label{alg:witness-trees}
\end{algorithm}

We can now describe the recursive algorithm for solving \pMTES.
The pseudo-code is given in \cref{alg:witness-trees}.
As mentioned, it checks whether the current witness ensemble is sufficiently small and classifies the input and, if so, reports it as a solution.
Otherwise, it finds a dirty example~$e$ and tries all possibilities of reclassifying the example in a refinement of one of the trees in the current ensemble.
Then it continues recursively.
In \cref{algline:refinement-loop}, a one-step refinement of~$T$ is \emph{important} if it is obtained by introducing a new node $w$ that is labeled by a dimension~$i \in [d]$ in which $e$ and the witness $x$ of $e$'s leaf in $T$ differ, i.e., $e[i] \neq x[i]$, and by a threshold~$\delta \in \Thr(i)$ such that $\delta$ is between $e[i]$ and~$x[i]$.

The initial calls to \RefineEnsemble\ are made with the following $2^{\ell}$ witness ensembles~$\mathcal{C}$:
For each tree $T$ in $\mathcal{C}$ we pick an arbitraryf example and try both possibilities for whether $e$ is classified as $\lambda(e)$ or not (i.e.\ with the other class) and make $T$ to be a tree consisting of a single leaf labeled by the corresponding class and with $e$ as its witness.
This concludes the description of the algorithm.
For the algorithm for \pMMTES\ we replace $S$ by $s$ and we modify \cref{algline:wt-size-check} to check that the size of the largest tree is larger than $s$ instead.%

To achieve our claimed running time for \pMTES, we need to ensure that $\ell \leq S$.
We claim that we can assume this without loss of generality: 
Note that, if $\ell > S$, then necessarily in the solution ensemble there are trivial trees without any cuts, that is, trees that classify all examples as $\lpos$ or all examples as $\lneg$.
Now with a factor $\ell$ in the running time, we may determine how many such trivial trees of either type there are:
For each integer $i$ from~$0$ to~$\ell - S$ we postulate that there are $i$ trivial trees with a single $\lpos$ leaf and~$\ell - S - i$ with a single $\lneg$ leaf and carry out the algorithm described above to determine the $\ell - S$ remaining nontrivial trees.
Such additional trivial trees can easily be incorporated into the algorithm described above.
We will thus assume that each solution tree has at least one inner node and thus that~$\ell \leq S$.

\begin{proof}[of \cref{thm:witness-tree-algo}]
	We now show that the algorithm described above achieves the required running time and that it is correct.
	
	For the running time, observe that, after one of the $2^{\ell}$ initial calls, the algorithm describes a search tree in which each node corresponds to a call to \RefineEnsemble.
	The depth of this tree is at most $S$ for \pMTES\ and at most~$s\ell$ for \pMMTES\ because in each call at least one refinement is made and thus the overall size increases by at least one.
	We claim that each search-tree node has at most $\delta D (S + \ell)$ or $\delta D s \ell$ children, respectively.
	To see this, we show that the total number of refinements of $\mathcal{C}$ considered in \cref{algline:refinement-loop} is bounded by that number: Each such refinement is specified (1) by a new root or an edge on a root-leaf path of a tree in~$\mathcal{C}$, (2) a dimension in the newly introduced node, and (3) a threshold in the newly introduced node.
	
	For (3) there are at most $D$ possibilities.
	
	For (2) there are at most $\delta$ possibilities: Observe that $e$ has a different class label than the witness $w$ of its old leaf.
	Thus, there are at most $\delta$ dimensions in which $e$ and $w$ differ.
        Since the refinements considered are important, they thus have one of these at most $\delta$ dimensions.

	For (1), the number of possibilities can be bounded as follows. In \pMTES, if we are adding a new root, there are at most $\ell$ choices. If we are subdividing an edge $f$, then $f$ is on a root-leaf path in one of the trees in the current ensembles to the leaf of $e$.
        In total, these paths can contain at most $S$ inner nodes and thus at most $S$ edges.
        In \pMMTES, there are at most $\ell$ ways to choose the tree $T$ to refine.
        Afterwards, since the leaf of $e$ is uniquely defined, we need to specify whether we introduce a new root or which edge we subdivide on the root-leaf path $P$ to $e$'s leaf in $T$.
        Note that there are at most $s - 1$ inner nodes of $T$ in $P$ (otherwise, after refining the tree would exceed the size threshold).
        Thus, there are at most $s$ possibilities to introduce a new root or choose an edge for subdivision.
        There are thus at most $s \ell$ refinements to consider.
	
	Thus, the overall search tree has size at most $(\delta D (S + \ell))^S \leq (2\delta D S)^S$ (resp.\ $(\delta D s \ell)^{s \ell}$).
	Accounting for the $2^{\ell} \leq 2^S$ initial calls and noticing that the operations in one search-tree node take $\Oh(Sdn)$ time yields the claimed running time.
	
	It remains to prove the correctness.
	Clearly, if the algorithm returns something different from $\bot$, then it is a tree ensemble that classifies $(E, \lambda)$ and is of the required size.
	Now assume that there is a tree ensemble that classifies $(E, \lambda)$ and is of the required size.
	We show that the algorithm will not return~$\bot$.
	We say that a witness-tree ensemble $\mathcal{C}$ is \emph{good} if there is a tree ensemble $\mathcal{C}^{\star}$ that refines $\mathcal{C}$, classifies~$(E, \lambda)$, and is of the required size.
	We claim that (1)~one of the ensembles $\mathcal{C}$ of an initial call to \RefineEnsemble\ is good and (2)~that if~$\mathcal{C}$ in a call to \cref{alg:witness-trees} is good, then either it classifies $(E, \lambda)$ or in at least one recursive call $\RefineEnsemble(\mathcal{C}', \cdot, \cdot)$ the tree ensemble $\mathcal{C}'$ is good.
	Observe that it is sufficient to prove both claims.
	
	\newcommand{\sol}{\ensuremath{\mathcal{C}^{\star}}}
	As to Claim (1): Consider a solution \sol\ and the witnesses that were chosen arbitrarily for~$\mathcal{C}$ before the initial calls to \RefineEnsemble.
	Fix an arbitrary mapping of trees between~$\mathcal{C}$ and \sol.
	Consider a tree $T$ in $\mathcal{C}$ and its corresponding tree $T^{\star}$ in \sol.
	Observe that $T^{\star}$ has a leaf $q$ such that $E[T^{\star}, q]$ contains the (single) witness of $T$.
	Label $q$ with this witness.
	For each remaining leaf, pick an arbitrary example that is classified in this leaf and label it as its witness.
	(Note that, without loss of generality, there is at least one example in each leaf because if there is a leaf without an example then we can find a smaller solution ensemble.)
	Observe that the witness tree resulting from $T^{\star}$ by this labeling is a refinement of $T$ if the leaves containing the single witness of $T$ have the same class label.
	Since we try all possible class labels for the leaf in $T$ before the initial calls to \RefineEnsemble, eventually we obtain that $T^{\star}$ is a refinement of $T$ and indeed this holds for all pairs of mapped trees of $\mathcal{C}$ and \sol.
	Thus, $\mathcal{C}$ is good.
	
	\looseness=-1
	As to Claim (2): Assume that the witness-tree ensemble $\mathcal{C}$ is good and let \sol\ be a corresponding witness-tree ensemble.
	If $\mathcal{C}$ classifies $(E, \lambda)$ then there is nothing to prove.
	Otherwise, there is at least one dirty example for $\mathcal{C}$.
	Let $e$ be the dirty example picked by the algorithm in \cref{algline:pick-dirty}.
	Consider the classes assigned to $e$ by the trees in $\mathcal{C}$ and those assigned by the corresponding refinements in \sol.
	Observe that, since \sol\ classifies $e$ correctly, there is at least one tree $T \in \mathcal{C}$ such that $T$ classifies $e$ incorrectly and $T$'s refinement $R \in \sol$ classifies $e$ correctly.
	In a refinement, the class of a witness is never changed and thus $e$ is not a witness in~$T$.
	Hence, the for loop in \cref{algline:tree-loop} selects the tree $T$ in one iteration.
	
	We now claim that the loop in \cref{algline:refinement-loop} will select a refinement~$T'$ such that~$R$ (thought of as a non-witness decision tree) is a refinement of $T'$.
	Consider a sequence of one-step refinements $T = T_1, T_2, \ldots, T_k = R$.
        We first show that we can assume that $R$ has $e$ as a witness.
        Assume that this is not the case.
        Consider $e$'s leaf in $R$ and the witness $x$ of this leaf.
        Let $T_i$ be the tree in which witness $x$ has been introduced.
        Since one-step refinements only shrink the sets of examples that are classified in leaves other than the newly introduced one, the leaf of example $e$ in $T_i$ is also $x$'s leaf.
        Thus, we may replace witness $x$ by $e$ in $T_i$ and in every tree thereafter.
        Thus, we may assume that $e$ is a witness in $R$.
        Indeed, by \cref{lem:reorder-refinements} we may assume that $e$ is introduced in $T_2$.
	Finally, observe that we may assume that $T_2$ is important and thus $T_2$ equals $T'$, the refinement selected by the algorithm in \cref{algline:refinement-loop}.
	Thus, the refined ensemble $\mathcal{C}'$ constructed from $T'$ is good, as claimed.
\end{proof}

Recall that \pDTSlong\ (\pDTS) is the special case of \pMMTES\ in which $\ell = 1$.
Thus we have the following.
\begin{corollary}
	\pDTSlong\ can be solved in $\Oh((\delta D s)^s \cdot sdn)$ time.
\end{corollary}
This improves on the running time for \pDTS\ given by \cite{OrdyniakS21} (see their main theorem, Theorem~8).
% In the meantime it has been shown that one can drop the dependency on $D$ in the exponential running-time factor at the cost of a worse dependency on $s$ and $\delta$, however~\cite{EibenOrdyniakPaesaniSzeider23}.

The following theorem shows that the exponent in our running time cannot be improved.
\begin{theorem}\label{thm:witness-tree-algo-tight}
	Solving \pMTES\ in $(\delta D S)^{o(S)} \cdot \poly$ time would contradict the \ethlong, even if~$D = 2$ and $\ell = 1$.
	Furthermore, \pMTES\ cannot be solved in $(2-\varepsilon)^{n}$~time even if~$D=2$ and~$\ell=1$ for any~$\varepsilon>0$ unless the Set Cover Conjecture is wrong.
\end{theorem}

\begin{proof}
	This follows from a reduction from the \pHS\ problem to \pDTS\ given by \cite{OrdyniakS21}.
	For completeness, we repeat the construction here, but for clarity reasons we instead reduce from \pDS.
	In \pDS\ we are given a graph~$G$ and an integer~$k$.
	We want to decide whether there is a vertex subset $W$ of size at most $k$ such that each vertex is in~$W$ or has a neighbor in $W$.
	For convenience put $V(G) = \{v_1, v_2, \ldots, v_n\}$.
	We construct an instance~$((E, \lambda), 1, S)$ of \pMTES\ as follows.
	We put $\ell = 1$ and $S = k$.
	For the training data, 
	there are~$|V(G)|$ dimensions $1, 2, \ldots, n$.
	The domain in each dimension is~$\{0, 1\}$.
	There is one $\lneg$ example~$x = (0, 0, \ldots, 0)$.
	For each vertex~$v_i \in V(G)$, there is one $\lpos$ example $e_i \in E$ such that~$e_i[j] = 1$ if $i = j$ or if~$v_i$ and $v_j$ are adjacent; otherwise~$e_i[j] = 0$.
	It is not hard to check (see~\cite{OrdyniakS21}) that $(G, k)$ and $((E, \lambda), 1, S)$ are equivalent.
	By \cite{chen_strong_2006} (Theorem~5.4), we know that an algorithm with running time $f(k) \cdot n^{o(k)}$ for \pDS\ would imply FPT${}={}$W[1] which would imply that the \ethlong\ is false~\cite{CyFoKoLoMaPiPiSa2015}.
	To finish the proof, observe that an $\Oh((\delta D S)^{o(S)} \cdot \poly)$-time algorithm for \pMTES\ would imply an $f(k) \cdot n^{o(k)}$-time algorithm for \pDS.
	
	To obtain the no $(2-\varepsilon)^{n}$~time lower bound, we use a similar reduction, but reduce from \textsc{Set Cover}, where the input is a universe~$\mathcal{U}$ of size~$n$ and a set family~$\mathcal{S}$.
	Furthermore, we swap the dimensions and the examples, that is, we create one dimensions per set in~$\mathcal{S}$, and one $\lpos$~example for each element in~$\mathcal{U}$ and additionally one $\lneg$~example.
	Hence, we have exactly $n+1$~elements.
	Since \textsc{Set Cover} cannot be solved in $(2-\varepsilon)^{n}$~time for any~$\varepsilon >0$~\cite{CyganDLMNOPSW16} if the Set Cover Conjecture is true, we obtain the desired lower bound for the number of examples.
\end{proof}

\section{Tight Exponential-Time Algorithm}
\label{sec:tight-exponential}

\subsection{An Efficient Algorithm for a Small Number of Examples}
\label{sec-algos-for-n}

We now give an algorithm that solves \pMTES{} in~$(\ell+1)^{n}\cdot \poly$~time, where~$n \coloneqq \abs{E}$ is the number of examples. This running time is single-exponential in~$n$ for every fixed number of trees. More importantly, we show that this running time is essentially optimal. To obtain the algorithm, we first show how to compute in a suitable running time the sizes of smallest trees for essentially all possible classification outcomes of a decision tree.     
\begin{lemma}
\label{lem-blue-correct}
	Given a training data set~$(E,\lambda)$ one can compute in $\Oh(3^n \cdot Ddn)$~time for all~$E' \subseteq E$ the size of a smallest decision tree~$T$ such that~$T[e]=\lpos$ if and only if~$e\in E'$.
\end{lemma}

\begin{proof}
  % \todo[inline]{F: I think the notation is not that nice: In the definition $E^b$ and $E^r$ are arbitrary subsets of the examples and in the initialization it is implicitly assumed that~$E^b$ are blue examples.}
  We solve the problem using dynamic programming over subsets of $E$.
  The dynamic-programming table has entries of the type~$Q[E^b,E^r]$ where~$E^b \subseteq E$ and~$E^r \subseteq E$ are disjoint subsets of examples.
  Each table entry stores the size of a smallest decision tree $T$ such that, if we use $T$ to classify $E^b\cup E^r$, then exactly the examples of~$E^b$ receive the label~$\lpos$.
  (Note that the examples in~$E^b$ and~$E^r$ are not necessarily correctly classified.
  Since some trees in an ensemble may misclassify some examples, we need to allow for this possibility here.)
  We fill the table entries for increasing values of~$|E^b\cup E^r|$ as follows.
  
  We initialize the table by setting 
  $$ Q[E^b,\emptyset]\coloneqq  0
  \text{ and } Q[\emptyset,E^r]\coloneqq  0$$ for all~$E^b\subseteq E$ and~all~$E^r\subseteq E$. 
  This is correct since in these cases, a decision tree without cuts and only one leaf with the appropriate class label suffices.
  
  The recurrence to fill the table when~$E^b$ and~$E^r$ are nonempty is
  \begin{align*}
    Q[E^b,E^r]\coloneqq \min_{i\in [d], h \in \Thr(i)} Q[E^b[f_i\le h],E^r[f_i\le h]] + \\  Q[E^b[f_i> h],E^r[f_i> h]] + 1.
  \end{align*}
  Recall that~$\Thr(i)$ denotes some minimum-size set of thresholds that distinguishes between all values of the examples in the~$i$th dimension. In other words, for each pair of elements~$e$ and~$e'$ with~$e[i] < e'[i]$, there is at least one value~$h\in \Thr(i)$ such that~$e[i] < h \le e'[i]$. Moreover, we only consider those cases where $E^b[f_i\le h]\cup E^r[f_i\le h]\neq \emptyset$ and $E^b[f_i> h]\cup E^r[f_i> h]\neq \emptyset$. That is, we consider only the case that the cut gives two nonempty subtrees. This ensures that the recurrence only considers table entries with smaller set sizes.
  
  The idea behind the recurrence is that we consider all possibilities for the cut at the root and then use the smallest decision trees to achieve the desired labeling for the two resulting subtrees. The size of the resulting tree is the size of the two subtrees plus one, for the additional root vertex.
  The formal correctness proof is standard and omitted.
  
  The running time bound can be seen as follows. 
We need $\Oh(Ddn)$~time to read the input.  
  The number of table entries is~$\Oh(3^n)$ since each entry corresponds to a 3-partition of~$E$ into~$E^b$,~$E^r$, and~$E\setminus (E^b\cup E^r)$. Each entry can be evaluated in $\Oh(Dd)$~time for each of the $d$~dimension we check each of the at most $D$~thresholds.
\end{proof}

We now briefly turn to decision trees instead of ensembles:
The above lemma directly gives an algorithm for \pDTS{} with running time~$\Oh(3^n \cdot Ddn)$.
We can improve on that by slightly modifying the table as follows.

\begin{theorem}\label{cor:single-exponential-decision-trees}
  \pDTSlong\ can be solved in $\Oh(2^n \cdot Ddn)$~time.
\end{theorem}
\begin{proof}
  We use the algorithm from \Cref{lem-blue-correct} to compute the table $Q$ but with the following modification:
  We restrict the table such that it contains only those entries $Q[E^b,E^r]$ where~$E^b$ is a subset of the \lpos\ examples in the input and $E^r$ is a subset of the \lneg\ examples in the input.
  Otherwise the initialization and recurrence are the same.
  To find the minimum size of a tree that classifies the input training data set, we simply look up~$Q[\{e \in E \mid \lambda(e) = \lpos\}, \{e \in E \mid \lambda(e) = \lneg\}]$.
  
  The initialization and recurrence remain correct because the solution trees for \pDTS\ do not have misclassifications.
  As for the running time, observe that the number of table entries is~$\Oh(2^a \cdot 2^b)$, where $a$ is the number of \lpos\ examples in the input and $b$ the number of \lneg\ examples in the input.
  This yields the claimed bound because $a + b = n$.
\end{proof}

Note that the running time of \Cref{cor:single-exponential-decision-trees} cannot be improved unless standard complexity assumption are wrong:
According to \Cref{thm:witness-tree-algo-tight}, any algorithm with running time~$(2-\varepsilon)^n$ for some~$\varepsilon>0$ contradicts the Set Cover Conjecture.

We now go back to ensembles and use the algorithm from \cref{lem-blue-correct} as a subroutine in an algorithm for~\pMTESlong.
\begin{theorem}\label{thm:exptime-algo}
	For~$\ell>1$, one can solve \pMTESlong\ in $(\ell + 1)^n \cdot Ddn$ time.
\end{theorem}
\begin{proof}
	We use again dynamic programming. It is not sufficient to use subsets of elements that are classified correctly. Instead, we build the solution by iteratively adding trees and storing for each example~$e$ how often~$e$ is classified correctly.
	
	To store subsolutions, we use a table~$R$ with entries of the type~$R[c,j]$ where~$c$ is a length-$n$ integer vector where each~$c_i$ is an integer in $[0,\lceil \ell/2\rceil]$ for each~$i\in [n]$ and~$j\in [\ell]$. An entry~$R[c,j]$ stores the smallest total size of any set of~$j$ decision trees such that each element~$e_i$ is classified correctly exactly~$c_i$~times if~$c_i<\lceil \ell/2\rceil$ and at least $c_i$~times if~$c_i = \lceil \ell/2\rceil$. 
	The distinction between~$c_i<\lceil \ell/2\rceil$ and~$c_i = \lceil \ell/2\rceil$ allows us to assign only one value of~$c_i$ to the situation that~$e_i$ is already correctly classified irrespective of the other trees of the ensemble.
	
	The first step of the algorithm is to compute for all~$E'\subseteq E$ the smallest size of any decision tree~$T$ assigning the \lpos{} label exactly to all~$e\in E'$. From this information, we can directly compute for all~$E'\subseteq E$, the size of a smallest decision tree that classifies all examples in~$E'$ correctly and all examples in~$E\setminus E'$ incorrectly. We will store these sizes in table entries~$Q[E']$.
	
	Now, we initialize~$R$ for~$j=1$, by setting
	
	$$R[c,1] \coloneqq 
	\begin{cases}
	Q[E'] & \exists E'\subseteq E: c=\indic{E'},\\
	+\infty & \text{otherwise.} 
	\end{cases}
	$$ 
	
	Here, we let $\indic{E'}$ denote the indicator vector for~$E'$. Now, for~$j>1$, we use the  recurrence
	\begin{equation}
	R[c,j] \coloneqq \min_{E'\subseteq E, c': c'
		\oplus \indic{E'}=c} R[c',j-1] + Q[E'].\label{eq:ell-recurrence}
	\end{equation}
	
	Here,~$\oplus$ is a truncated addition, that is, for the~$i$th component of~$c'$, we add~1 if this component is strictly smaller than~$\lceil \ell/2\rceil$. If for some~$R[c,j]$ the minimum ranges over an empty set, then we set~$R[c,j] \coloneqq + \infty$.
	
	The idea of the recurrence is simply that the~$j$th tree classifies some element set~$E'$ correctly and that this increases the number of correct classifications for all elements of~$E'$.
	The smallest size of a tree ensemble with~$\ell$ trees to correctly classify~$E$ can then be found in~$R[c^*,
	\ell]$ where we let~$c^*$ denote the length-$n$ vector where each component has value~$\lceil \ell/2\rceil$.
	The formal proof is again standard and omitted.
	
	The table has size~$(\lceil \ell/2\rceil+1)^n\cdot \ell$. The bottleneck in the running time to fill the table is the time needed for evaluating the~$\min$ in Equation~\ref{eq:ell-recurrence}. A straightforward estimation gives a time of~$2^n\cdot \lceil \ell/2 \rceil^n$ for each entry since we consider all possible subsets~$E'$ of~$E$ and possibly all vectors~$c$'. Instead, we may fill the table entries also in a forward direction, that is, for each~$c'$ and each~$E'\subseteq E$, we compute~$c'
	\oplus \indic{E'}$ and update the table entry for~$R[c,j]$ if~$R[c',j-1]+Q[E']$ is smaller than the current entry of~$R[c,j]$. This way, the total time for updating table entries is~$2^n\cdot \lceil \ell/2 \rceil^n\le (\ell+1)^n$ since we consider~$2^n$ possible choices for~$E'$ at each vector~$c'$ and directly derive the corresponding~$c$ for each choice. 
	The overall time bound follows from the observation that the $\Oh(3^n\cdot Ddn)$~time needed for the preprocessing is upper-bounded by~$(\ell+1)^n $ since~$\ell>1$.
\end{proof}

\subsection{A Matching Lower Bound}

We now show that, under a standard conjecture in complexity theory, the running time of the algorithm of \Cref{thm:exptime-algo} cannot be improved substantially.

We show this by a reduction from \abColor.
Here, one is given a graph~$G$ and two integers~$a$ and~$x$, and wants to assign each vertex in~$V(G)$ a set of~$x$ out of $a$~colors such that each two adjacent vertices receive disjoint color sets.
Unless the ETH fails, \abColor{} cannot be solved in $f(x)\cdot 2^{o(\log(x))\cdot n}$~time, where $n$ is the number of vertices even if~$a=\Theta(x^2 \log x)$~\cite{bonamy_tight_2019}.
Observe that in a solution for \abColor, the vertices having some color~$c$ form an independent set in~$G$.
Our aim is to transfer this lower bound to \pMTESlong.
To do so, we take a detour and first transfer this lower bound to \textsc{Set Multicover} which might be of independent interest.

In \textsc{Set Multicover} the input is a universe~$\mathcal{U}$, a family~$\mathcal{F}$ of subsets of~$U$, an integer~$a$, and each element~$u\in\mathcal{U}$ has a \emph{demand}~$\dem(u)$.
The task is to find a \emph{set multicover}, that is, a subfamily~$\mathcal{F}'\subseteq \mathcal{F}$ of size exactly~$a$ such that for each~$u\in\mathcal{U}$ at least~$\dem(u)$ sets of~$\mathcal{F}'$ contain element~$u$.

\probdef{\textsc{Set Multicover}}
{A universe~$\mathcal{U}$, a family of subsets~$\mathcal{F}$ of~$U$, a demand function~$\dem:U\to\mathds{N}$, and an integer~$a$.}
{Is there a set multicover of size exactly~$a$?}

\begin{proposition}
	\label{thm:superexponential-lower-bound-set-multi-cover}
	Solving \textsc{Set Multicover} in $f(x) \cdot 2^{o(\log x)\cdot |\mathcal{U}|}\cdot\poly$~time would contradict the \ethlong, even if all demands are equal to~$x$ and if~$a\in\Theta(x^2\log x)$.
\end{proposition}
\begin{proof}
	We provide a reduction from \abColor, which cannot be solved in $f(x)\cdot 2^{o(\log x)\cdot n}$~time, where $n$ is the number of vertices even if~$a=\Theta(x^2 \log x)$ unless the ETH fails~\cite{bonamy_tight_2019}. 
	Let~$(G,a,x)$ be an instance of \abColor. 
	We construct an equivalent instance~$(\mathcal{U},\mathcal{F},a,x)$ of \textsc{Set Multicover} where each demand is equal to~$x$ as follows.
	Each element in the universe~$\mathcal{U}$ corresponds to a vertex of~$V(G)$.
	Finally, for each maximal independent set~$I$, we add a set~$F_I$.
	
	Now, we show the correctness.
	Consider a solution of~$(G,a,x)$ and let~$C_i$ for~$i\in[a]$ be the vertices having color~$i$.
	By definition~$G[C_i]$ is an independent set and thus there exists a set~$F_i$ containing all vertices of~$C_i$.
	Note that the family~$\{F_i:i\in[a]\}$ is a set multicover:
	Each vertex received at least $x$~colors and thus each element is covered at least $x$~times.
	Conversely, consider a set multicover~$\{F_i:i\in[a]\}$.
	By definition,~$F_i$ corresponds to an independent set of~$G$.
	Hence, by assigning color~$i$ to all vertices of~$F_i$ we obtain a solution for~$(G,a,x)$.
	
	Observe that since~$G$ may contain at most $3^{n/3}$~maximal independent sets~\cite{moonM1965}, we add at most $3^{n/3}$~sets.
	Our reduction  can be carried out by iterating over all vertex sets in $\Oh(n\cdot 2^n)$~time.
	Now, if \textsc{Set Multicover} has an algorithm with running time~$f(x)\cdot 2^{o(\log x)\cdot |\mathcal{U}|}\cdot\poly$ and since the number of sets we have to choose is~$a$, this implies an algorithm with running time~$f(a)\cdot 2^{o(\log a)\cdot n}\cdot\poly+n\cdot 2^n$ for \abColor.
	Since~$a\in\Theta(x^2 \log x)$, this then implies an algorithm running in $f(x)\cdot 2^{o(\log x^3)\cdot n}=f(x)\cdot 2^{o(\log x)\cdot n}$~time, a contradiction to the ETH~\cite{bonamy_tight_2019}.
\end{proof}

Now, we present our reduction for \pMTESlong.
In our reduction, we have a \emph{choice dimension} for each set~$F$ in the family~$\mathcal{F}$.
Furthermore, we have an \emph{element example} for each element of~$\mathcal{U}$.
Also, we set~$\ell\coloneqq 2a+1$ and~$S\coloneqq 2a+1$.
The main idea of our reduction is that there are exactly $a$~many trees cutting a choice dimension such that each element example is correctly classified in at least~$x$ of these trees.
We achieve this by adding some \emph{dummy dimensions} and further element examples so that each correct tree ensemble consists of exactly $\ell$~trees having exactly one inner node, as we show.

\begin{theorem}
	\label{thm:superexponential-lower-bound}
	Solving \pMTESlong\ in $f(\ell) \cdot 2^{o(\log \ell)\cdot n}\cdot \poly$~time would contradict the \ethlong.
\end{theorem}

\begin{proof}
	We reduce from \textsc{Set Multicover} where each demand is~$x$, which cannot be solved in $f(x)\cdot 2^{o(\log x)\cdot |\mathcal{U}|}$~time, even if~$k=\Theta(x^2 \log x)$ unless the ETH fails, according to \Cref{thm:superexponential-lower-bound-set-multi-cover}.

	%\probdef{\abColor}
	%{A graph~$G$, and two integers~$a$ and~$x$.}
	%{Is there an~$x$-folding function~$\fol$ for~$(G,a,x)$?}
	
	%For two integers~$a$ and~$x$, by~$\binom{[a]}{x}$ we denote all subsets of size~$x$ of~$[a]$. 
	%A function~$\fol: V(G)\to \binom{[a]}{x}$ is \emph{$x$-folding} if for each edge~$uv$ of~$G$ we have~$\fol(u)\cap \fol(v)=\emptyset$.
	%In other words,~$\fol$ is~$x$-folding for~$(G,a,x)$, if each vertex is assigned a set of $x$~colors such that each two adjacent vertices receive disjoint color sets.
	
	\begin{figure*}[t]
		\centering
		\begin{tikzpicture}[scale=0.41]
		\node (A) at (-2,4.3) {};
		\node (B) at (27.7,4.3) {};
		\path [-,line width=0.3mm] (A) edge (B);
		\node (A) at (-1,1.3) {};
		\node (B) at (27.7,1.3) {};
		\path [-,line width=0.3mm] (A) edge (B);
		\node (A) at (-1,-1.7) {};
		\node (B) at (27.7,-1.7) {};
		\path [-,line width=0.3mm] (A) edge (B);
		\node (A) at (-1,-2.7) {};
		\node (B) at (27.7,-2.7) {};
		\path [-,line width=0.3mm] (A) edge (B);
		\node (A) at (-2,-5.7) {};
		\node (B) at (27.7,-5.7) {};
		\path [-,line width=0.3mm] (A) edge (B);
		
		\node (C) at (26.5,-7.7) {};
		\node (D) at (26.5,12.5) {};
		\path [-,line width=0.3mm] (C) edge (D);
		\node (C) at (22.6,-7.7) {};
		\node (D) at (22.6,12.5) {};
		\path [-,line width=0.3mm] (C) edge (D);
		\node (C) at (15.5,-7.7) {};
		\node (D) at (15.5,12.5) {};
		\path [-,line width=0.3mm] (C) edge (D);
		\node (C) at (9.65,-7.7) {};
		\node (D) at (9.65,12.5) {};
		\path [-,line width=0.3mm] (C) edge (D);
		\node (C) at (4.5,-7.7) {};
		\node (D) at (4.5,12.5) {};
		\path [-,line width=0.3mm] (C) edge (D);
		\node (C) at (3.5,-7.7) {};
		\node (D) at (3.5,12.5) {};
		\path [-,line width=0.3mm] (C) edge (D);
		\node (C) at (0.5,-7) {};
		\node (D) at (0.5,12.5) {};
		\path [-,line width=0.3mm] (C) edge (D);
		\tiny
		\node[label=above:{\rotatebox{90}{dummy dimensions}}] at (-1.9,6){};
		\node[label=above:{\rotatebox{90}{choice dimensions}}] at (-1.9,-3.7){};
		\node[label=above:{forcing examples}] at (7,-7.9){};
		\node[label=above:{element examples}] at (1.0,-7.9){};
		\node[label=above:{choosing examples}] at (12.5,-7.9){};
		\node[label=above:{verifying examples}] at (19,-7.9){};
		\node[label=above:{test examples}] at (24.5,-7.9){};
		
		% color boxes according to classification
		\draw[pattern=north east lines, pattern color=blue, opacity=0.4] (0.5,-6.8) rectangle (4.5,-5.7);
		\draw[pattern=north west lines, pattern color=red, opacity=0.4] (4.5,-6.8) rectangle (9.65,-5.7);
		\draw[pattern=north east lines, pattern color=blue, opacity=0.4] (9.65,-6.8) rectangle (15.5,-5.7);
		\draw[pattern=north west lines, pattern color=red, opacity=0.4] (15.5,-6.8) rectangle (26.5,-5.7);
		\draw[pattern=north east lines, pattern color=blue, opacity=0.4] (26.5,-6.8) rectangle (27.5,-5.7);
		
		\tiny
		% definition examples
		\node[label=above:{$b_\val$}] at (4,-7){};
		\node[label=above:{$r_1$}] at (5,-7){};
		\node[label=above:{$r_2$}] at (6,-7){};
		\node[label=above:{$r_3$}] at (7,-7){};
		\node[label=above:{$\cdots$}] at (7.8,-7){};
		\node[label=above:{$r_{a+1}$}] at (8.9,-7){};
		\node[label=above:{$b_u$}] at (1,-7){};
		\node[label=above:{$b_v$}] at (2,-7){};
		\node[label=above:{$b_w$}] at (3,-7){};
		\node[label=above:{$b^1_1$}] at (10,-7){};
		\node[label=above:{$b^1_2$}] at (11,-7){};
		\node[label=above:{$\cdots$}] at (11.8,-7){};
		\node[label=above:{$b^1_{a+1}$}] at (12.8,-7){};
		\node[label=above:{$b^2_1$}] at (13.8,-7){};
		\node[label=above:{$b^2_{a+1}$}] at (14.8,-7){};
		\node[label=above:{$r^1_{1,2}$}] at (16.1,-7){};
		\node[label=above:{$r^1_{1,3}$}] at (17.1,-7){};
		%\node[label=above:{$\cdot\cdot$}] at (17.75,-7){};
		\node[label=above:{$r^1_{a,a+1}$}] at (18.6,-7){};
		\node[label=above:{$r^a_{1,2}$}] at (20,-7){};
		%\node[label=above:{$\cdot\cdot$}] at (20.75,-7){};
		\node[label=above:{$r^a_{a,a+1}$}] at (21.6,-7){};
		\node[label=above:{$r^1$}] at (23,-6.9){};
		\node[label=above:{$r^2$}] at (24,-6.9){};
		\node[label=above:{$\cdots$}] at (25,-6.9){};
		\node[label=above:{$r^a$}] at (26,-6.9){};
		\node[label=above:{$b_\enf$}] at (27,-7){};
		
		\node[label=above:{$d_1$}] at (0,4){};
		\node[label=above:{$d_2$}] at (0,5){};
		\node[label=above:{$d_3$}] at (0,6){};
		\node[label=above:{$\vdots$}] at (0,7){};
		\node[label=above:{$d_{a+1-x}$}] at (-0.6,8){};
		\node[label=above:{$d_{a+2-x}$}] at (-0.6,9){};
		\node[label=above:{$\vdots$}] at (0,10){};
		\node[label=above:{$d_{a+1}$}] at (-0.3,11){};
		\node[label=above:{$d^1_{F_1}$}] at (0,1){};
		\node[label=above:{$d^1_{F_2}$}] at (0,2){};
		\node[label=above:{$d^1_{F_3}$}] at (0,3){};
		\node[label=above:{$d^2_{F_1}$}] at (0,-2){};
		\node[label=above:{$d^2_{F_2}$}] at (0,-1){};
		\node[label=above:{$d^2_{F_3}$}] at (0,0){};
		\node[label=above:{$\vdots$}] at (0,-3){};
		\node[label=above:{$d^a_{F_1}$}] at (0,-6){};
		\node[label=above:{$d^a_{F_2}$}] at (0,-5){};
		\node[label=above:{$d^a_{F_3}$}] at (0,-4){};
		
		% e_u
		\node[label=above:{$0$}] at (1,11.1){};
		\node[label=above:{$\vdots$}] at (1,10){};
		\node[label=above:{$0$}] at (1,9.1){};
		\node[label=above:{$0$}] at (1,8.1){};
		\node[label=above:{$\vdots$}] at (1,7){};
		\node[label=above:{$1$}] at (1,6.1){};
		\node[label=above:{$1$}] at (1,5.1){};
		\node[label=above:{$1$}] at (1,4.1){};
		\node[label=above:{$0$}] at (1,3.1){};
		\node[label=above:{$0$}] at (1,2.1){};
		\node[label=above:{$1$}] at (1,1.1){};
		\node[label=above:{$0$}] at (1,0.1){};
		\node[label=above:{$0$}] at (1,-0.9){};
		\node[label=above:{$1$}] at (1,-1.9){};
		\node[label=above:{$\vdots$}] at (1,-3){};
		\node[label=above:{$0$}] at (1,-3.9){};
		\node[label=above:{$0$}] at (1,-4.9){};
		\node[label=above:{$1$}] at (1,-5.9){};
		
		% e_v
		\node[label=above:{$0$}] at (2,11.1){};
		\node[label=above:{$\vdots$}] at (2,10){};
		\node[label=above:{$0$}] at (2,9.1){};
		\node[label=above:{$0$}] at (2,8.1){};
		\node[label=above:{$\vdots$}] at (2,7){};
		\node[label=above:{$1$}] at (2,6.1){};
		\node[label=above:{$1$}] at (2,5.1){};
		\node[label=above:{$1$}] at (2,4.1){};
		\node[label=above:{$1$}] at (2,3.1){};
		\node[label=above:{$1$}] at (2,2.1){};
		\node[label=above:{$0$}] at (2,1.1){};
		\node[label=above:{$1$}] at (2,0.1){};
		\node[label=above:{$1$}] at (2,-0.9){};
		\node[label=above:{$0$}] at (2,-1.9){};
		\node[label=above:{$\vdots$}] at (2,-3){};
		\node[label=above:{$1$}] at (2,-3.9){};
		\node[label=above:{$1$}] at (2,-4.9){};
		\node[label=above:{$0$}] at (2,-5.9){};
		
		% e_w
		\node[label=above:{$0$}] at (3,11.1){};
		\node[label=above:{$\vdots$}] at (3,10){};
		\node[label=above:{$0$}] at (3,9.1){};
		\node[label=above:{$0$}] at (3,8.1){};
		\node[label=above:{$\vdots$}] at (3,7){};
		\node[label=above:{$1$}] at (3,6.1){};
		\node[label=above:{$1$}] at (3,5.1){};
		\node[label=above:{$1$}] at (3,4.1){};
		\node[label=above:{$1$}] at (3,3.1){};
		\node[label=above:{$0$}] at (3,2.1){};
		\node[label=above:{$1$}] at (3,1.1){};
		\node[label=above:{$1$}] at (3,0.1){};
		\node[label=above:{$0$}] at (3,-0.9){};
		\node[label=above:{$1$}] at (3,-1.9){};
		\node[label=above:{$\vdots$}] at (3,-3){};
		\node[label=above:{$1$}] at (3,-3.9){};
		\node[label=above:{$0$}] at (3,-4.9){};
		\node[label=above:{$1$}] at (3,-5.9){};
		
		% e_val
		\node[label=above:{$1$}] at (4,11.1){};
		\node[label=above:{$\vdots$}] at (4,10){};
		\node[label=above:{$1$}] at (4,9.1){};
		\node[label=above:{$1$}] at (4,8.1){};
		\node[label=above:{$\vdots$}] at (4,7){};
		\node[label=above:{$1$}] at (4,6.1){};
		\node[label=above:{$1$}] at (4,5.1){};
		\node[label=above:{$1$}] at (4,4.1){};
		\node[label=above:{$0$}] at (4,3.1){};
		\node[label=above:{$0$}] at (4,2.1){};
		\node[label=above:{$0$}] at (4,1.1){};
		\node[label=above:{$0$}] at (4,0.1){};
		\node[label=above:{$0$}] at (4,-0.9){};
		\node[label=above:{$0$}] at (4,-1.9){};
		\node[label=above:{$\vdots$}] at (4,-3){};
		\node[label=above:{$0$}] at (4,-3.9){};
		\node[label=above:{$0$}] at (4,-4.9){};
		\node[label=above:{$0$}] at (4,-5.9){};
		
		% f_1
		\node[label=above:{$1$}] at (5,11.1){};
		\node[label=above:{$\vdots$}] at (5,10){};
		\node[label=above:{$1$}] at (5,9.1){};
		\node[label=above:{$1$}] at (5,8.1){};
		\node[label=above:{$\vdots$}] at (5,7){};
		\node[label=above:{$1$}] at (5,6.1){};
		\node[label=above:{$1$}] at (5,5.1){};
		\node[label=above:{$0$}] at (5,4.1){};
		\node[label=above:{$0$}] at (5,3.1){};
		\node[label=above:{$0$}] at (5,2.1){};
		\node[label=above:{$0$}] at (5,1.1){};
		\node[label=above:{$0$}] at (5,0.1){};
		\node[label=above:{$0$}] at (5,-0.9){};
		\node[label=above:{$0$}] at (5,-1.9){};
		\node[label=above:{$\vdots$}] at (5,-3){};
		\node[label=above:{$0$}] at (5,-3.9){};
		\node[label=above:{$0$}] at (5,-4.9){};
		\node[label=above:{$0$}] at (5,-5.9){};
		
		% f_2
		\node[label=above:{$1$}] at (6,11.1){};
		\node[label=above:{$\vdots$}] at (6,10){};
		\node[label=above:{$1$}] at (6,9.1){};
		\node[label=above:{$1$}] at (6,8.1){};
		\node[label=above:{$\vdots$}] at (6,7){};
		\node[label=above:{$1$}] at (6,6.1){};
		\node[label=above:{$0$}] at (6,5.1){};
		\node[label=above:{$1$}] at (6,4.1){};
		\node[label=above:{$0$}] at (6,3.1){};
		\node[label=above:{$0$}] at (6,2.1){};
		\node[label=above:{$0$}] at (6,1.1){};
		\node[label=above:{$0$}] at (6,0.1){};
		\node[label=above:{$0$}] at (6,-0.9){};
		\node[label=above:{$0$}] at (6,-1.9){};
		\node[label=above:{$\vdots$}] at (6,-3){};
		\node[label=above:{$0$}] at (6,-3.9){};
		\node[label=above:{$0$}] at (6,-4.9){};
		\node[label=above:{$0$}] at (6,-5.9){};
		
		% f_3
		\node[label=above:{$1$}] at (7,11.1){};
		\node[label=above:{$\vdots$}] at (7,10){};
		\node[label=above:{$1$}] at (7,9.1){};
		\node[label=above:{$1$}] at (7,8.1){};
		\node[label=above:{$\vdots$}] at (7,7){};
		\node[label=above:{$0$}] at (7,6.1){};
		\node[label=above:{$1$}] at (7,5.1){};
		\node[label=above:{$1$}] at (7,4.1){};
		\node[label=above:{$0$}] at (7,3.1){};
		\node[label=above:{$0$}] at (7,2.1){};
		\node[label=above:{$0$}] at (7,1.1){};
		\node[label=above:{$0$}] at (7,0.1){};
		\node[label=above:{$0$}] at (7,-0.9){};
		\node[label=above:{$0$}] at (7,-1.9){};
		\node[label=above:{$\vdots$}] at (7,-3){};
		\node[label=above:{$0$}] at (7,-3.9){};
		\node[label=above:{$0$}] at (7,-4.9){};
		\node[label=above:{$0$}] at (7,-5.9){};
		
		% ...
		\node[label=above:{$\cdots$}] at (8,11.1){};
		\node[label=above:{$\adots$}] at (8,10){};
		\node[label=above:{$\cdots$}] at (8,9.1){};
		\node[label=above:{$\cdots$}] at (8,8.1){};
		\node[label=above:{$\adots$}] at (8,7){};
		\node[label=above:{$\cdots$}] at (8,6.1){};
		\node[label=above:{$\cdots$}] at (8,5.1){};
		\node[label=above:{$\cdots$}] at (8,4.1){};
		\node[label=above:{$\cdots$}] at (8,3.1){};
		\node[label=above:{$\cdots$}] at (8,2.1){};
		\node[label=above:{$\cdots$}] at (8,1.1){};
		\node[label=above:{$\cdots$}] at (8,0.1){};
		\node[label=above:{$\cdots$}] at (8,-0.9){};
		\node[label=above:{$\cdots$}] at (8,-1.9){};
		\node[label=above:{$\adots$}] at (8,-3){};
		\node[label=above:{$\cdots$}] at (8,-3.9){};
		\node[label=above:{$\cdots$}] at (8,-4.9){};
		\node[label=above:{$\cdots$}] at (8,-5.9){};
		
		% f_{a+1}
		\node[label=above:{$0$}] at (9,11.1){};
		\node[label=above:{$\vdots$}] at (9,10){};
		\node[label=above:{$1$}] at (9,9.1){};
		\node[label=above:{$1$}] at (9,8.1){};
		\node[label=above:{$\vdots$}] at (9,7){};
		\node[label=above:{$1$}] at (9,6.1){};
		\node[label=above:{$1$}] at (9,5.1){};
		\node[label=above:{$1$}] at (9,4.1){};
		\node[label=above:{$0$}] at (9,3.1){};
		\node[label=above:{$0$}] at (9,2.1){};
		\node[label=above:{$0$}] at (9,1.1){};
		\node[label=above:{$0$}] at (9,0.1){};
		\node[label=above:{$0$}] at (9,-0.9){};
		\node[label=above:{$0$}] at (9,-1.9){};
		\node[label=above:{$\vdots$}] at (9,-3){};
		\node[label=above:{$0$}] at (9,-3.9){};
		\node[label=above:{$0$}] at (9,-4.9){};
		\node[label=above:{$0$}] at (9,-5.9){};
		
		% e^1_1
		\node[label=above:{$1$}] at (10,11.1){};
		\node[label=above:{$\vdots$}] at (10,10){};
		\node[label=above:{$1$}] at (10,9.1){};
		\node[label=above:{$1$}] at (10,8.1){};
		\node[label=above:{$\vdots$}] at (10,7){};
		\node[label=above:{$1$}] at (10,6.1){};
		\node[label=above:{$1$}] at (10,5.1){};
		\node[label=above:{$0$}] at (10,4.1){};
		\node[label=above:{$1$}] at (10,3.1){};
		\node[label=above:{$1$}] at (10,2.1){};
		\node[label=above:{$1$}] at (10,1.1){};
		\node[label=above:{$0$}] at (10,0.1){};
		\node[label=above:{$0$}] at (10,-0.9){};
		\node[label=above:{$0$}] at (10,-1.9){};
		\node[label=above:{$\vdots$}] at (10,-3){};
		\node[label=above:{$0$}] at (10,-3.9){};
		\node[label=above:{$0$}] at (10,-4.9){};
		\node[label=above:{$0$}] at (10,-5.9){};
		
		% e^1_2
		\node[label=above:{$1$}] at (11,11.1){};
		\node[label=above:{$\vdots$}] at (11,10){};
		\node[label=above:{$1$}] at (11,9.1){};
		\node[label=above:{$1$}] at (11,8.1){};
		\node[label=above:{$\vdots$}] at (11,7){};
		\node[label=above:{$1$}] at (11,6.1){};
		\node[label=above:{$0$}] at (11,5.1){};
		\node[label=above:{$1$}] at (11,4.1){};
		\node[label=above:{$1$}] at (11,3.1){};
		\node[label=above:{$1$}] at (11,2.1){};
		\node[label=above:{$1$}] at (11,1.1){};
		\node[label=above:{$0$}] at (11,0.1){};
		\node[label=above:{$0$}] at (11,-0.9){};
		\node[label=above:{$0$}] at (11,-1.9){};
		\node[label=above:{$\vdots$}] at (11,-3){};
		\node[label=above:{$0$}] at (11,-3.9){};
		\node[label=above:{$0$}] at (11,-4.9){};
		\node[label=above:{$0$}] at (11,-5.9){};
		
		% ...
		\node[label=above:{$\cdots$}] at (12,11.1){};
		\node[label=above:{$\adots$}] at (12,10){};
		\node[label=above:{$\cdots$}] at (12,9.1){};
		\node[label=above:{$\cdots$}] at (12,8.1){};
		\node[label=above:{$\adots$}] at (12,7){};
		\node[label=above:{$\cdots$}] at (12,6.1){};
		\node[label=above:{$\cdots$}] at (12,5.1){};
		\node[label=above:{$\cdots$}] at (12,4.1){};
		\node[label=above:{$\cdots$}] at (12,3.1){};
		\node[label=above:{$\cdots$}] at (12,2.1){};
		\node[label=above:{$\cdots$}] at (12,1.1){};
		\node[label=above:{$\cdots$}] at (12,0.1){};
		\node[label=above:{$\cdots$}] at (12,-0.9){};
		\node[label=above:{$\cdots$}] at (12,-1.9){};
		\node[label=above:{$\adots$}] at (12,-3){};
		\node[label=above:{$\cdots$}] at (12,-3.9){};
		\node[label=above:{$\cdots$}] at (12,-4.9){};
		\node[label=above:{$\cdots$}] at (12,-5.9){};
		
		% e^1_{a+1}
		\node[label=above:{$0$}] at (13,11.1){};
		\node[label=above:{$\vdots$}] at (13,10){};
		\node[label=above:{$1$}] at (13,9.1){};
		\node[label=above:{$1$}] at (13,8.1){};
		\node[label=above:{$\vdots$}] at (13,7){};
		\node[label=above:{$1$}] at (13,6.1){};
		\node[label=above:{$1$}] at (13,5.1){};
		\node[label=above:{$1$}] at (13,4.1){};
		\node[label=above:{$1$}] at (13,3.1){};
		\node[label=above:{$1$}] at (13,2.1){};
		\node[label=above:{$1$}] at (13,1.1){};
		\node[label=above:{$0$}] at (13,0.1){};
		\node[label=above:{$0$}] at (13,-0.9){};
		\node[label=above:{$0$}] at (13,-1.9){};
		\node[label=above:{$\vdots$}] at (13,-3){};
		\node[label=above:{$0$}] at (13,-3.9){};
		\node[label=above:{$0$}] at (13,-4.9){};
		\node[label=above:{$0$}] at (13,-5.9){};
		
		% e^2_1
		\node[label=above:{$1$}] at (14,11.1){};
		\node[label=above:{$\vdots$}] at (14,10){};
		\node[label=above:{$1$}] at (14,9.1){};
		\node[label=above:{$1$}] at (14,8.1){};
		\node[label=above:{$\vdots$}] at (14,7){};
		\node[label=above:{$1$}] at (14,6.1){};
		\node[label=above:{$1$}] at (14,5.1){};
		\node[label=above:{$0$}] at (14,4.1){};
		\node[label=above:{$0$}] at (14,3.1){};
		\node[label=above:{$0$}] at (14,2.1){};
		\node[label=above:{$0$}] at (14,1.1){};
		\node[label=above:{$1$}] at (14,0.1){};
		\node[label=above:{$1$}] at (14,-0.9){};
		\node[label=above:{$1$}] at (14,-1.9){};
		\node[label=above:{$\vdots$}] at (14,-3){};
		\node[label=above:{$0$}] at (14,-3.9){};
		\node[label=above:{$0$}] at (14,-4.9){};
		\node[label=above:{$0$}] at (14,-5.9){};
		
		% e^2_1
		\node[label=above:{$0$}] at (15,11.1){};
		\node[label=above:{$\vdots$}] at (15,10){};
		\node[label=above:{$1$}] at (15,9.1){};
		\node[label=above:{$1$}] at (15,8.1){};
		\node[label=above:{$\vdots$}] at (15,7){};
		\node[label=above:{$1$}] at (15,6.1){};
		\node[label=above:{$1$}] at (15,5.1){};
		\node[label=above:{$1$}] at (15,4.1){};
		\node[label=above:{$0$}] at (15,3.1){};
		\node[label=above:{$0$}] at (15,2.1){};
		\node[label=above:{$0$}] at (15,1.1){};
		\node[label=above:{$1$}] at (15,0.1){};
		\node[label=above:{$1$}] at (15,-0.9){};
		\node[label=above:{$1$}] at (15,-1.9){};
		\node[label=above:{$\vdots$}] at (15,-3){};
		\node[label=above:{$0$}] at (15,-3.9){};
		\node[label=above:{$0$}] at (15,-4.9){};
		\node[label=above:{$0$}] at (15,-5.9){};
		
		% f^1_{1,2}
		\node[label=above:{$1$}] at (16,11.1){};
		\node[label=above:{$\vdots$}] at (16,10){};
		\node[label=above:{$1$}] at (16,9.1){};
		\node[label=above:{$1$}] at (16,8.1){};
		\node[label=above:{$\vdots$}] at (16,7){};
		\node[label=above:{$1$}] at (16,6.1){};
		\node[label=above:{$0$}] at (16,5.1){};
		\node[label=above:{$0$}] at (16,4.1){};
		\node[label=above:{$1$}] at (16,3.1){};
		\node[label=above:{$1$}] at (16,2.1){};
		\node[label=above:{$1$}] at (16,1.1){};
		\node[label=above:{$0$}] at (16,0.1){};
		\node[label=above:{$0$}] at (16,-0.9){};
		\node[label=above:{$0$}] at (16,-1.9){};
		\node[label=above:{$\vdots$}] at (16,-3){};
		\node[label=above:{$0$}] at (16,-3.9){};
		\node[label=above:{$0$}] at (16,-4.9){};
		\node[label=above:{$0$}] at (16,-5.9){};
		
		% f^1_{1,3}
		\node[label=above:{$1$}] at (17,11.1){};
		\node[label=above:{$\vdots$}] at (17,10){};
		\node[label=above:{$1$}] at (17,9.1){};
		\node[label=above:{$1$}] at (17,8.1){};
		\node[label=above:{$\vdots$}] at (17,7){};
		\node[label=above:{$0$}] at (17,6.1){};
		\node[label=above:{$1$}] at (17,5.1){};
		\node[label=above:{$0$}] at (17,4.1){};
		\node[label=above:{$1$}] at (17,3.1){};
		\node[label=above:{$1$}] at (17,2.1){};
		\node[label=above:{$1$}] at (17,1.1){};
		\node[label=above:{$0$}] at (17,0.1){};
		\node[label=above:{$0$}] at (17,-0.9){};
		\node[label=above:{$0$}] at (17,-1.9){};
		\node[label=above:{$\vdots$}] at (17,-3){};
		\node[label=above:{$0$}] at (17,-3.9){};
		\node[label=above:{$0$}] at (17,-4.9){};
		\node[label=above:{$0$}] at (17,-5.9){};
		
		% ...
		\node[label=above:{$\cdots$}] at (18,11.1){};
		\node[label=above:{$\adots$}] at (18,10){};
		\node[label=above:{$\cdots$}] at (18,9.1){};
		\node[label=above:{$\cdots$}] at (18,8.1){};
		\node[label=above:{$\adots$}] at (18,7){};
		\node[label=above:{$\cdots$}] at (18,6.1){};
		\node[label=above:{$\cdots$}] at (18,5.1){};
		\node[label=above:{$\cdots$}] at (18,4.1){};
		\node[label=above:{$\cdots$}] at (18,3.1){};
		\node[label=above:{$\cdots$}] at (18,2.1){};
		\node[label=above:{$\cdots$}] at (18,1.1){};
		\node[label=above:{$\cdots$}] at (18,0.1){};
		\node[label=above:{$\cdots$}] at (18,-0.9){};
		\node[label=above:{$\cdots$}] at (18,-1.9){};
		\node[label=above:{$\adots$}] at (18,-3){};
		\node[label=above:{$\cdots$}] at (18,-3.9){};
		\node[label=above:{$\cdots$}] at (18,-4.9){};
		\node[label=above:{$\cdots$}] at (18,-5.9){};
		
		% f^1_{a,a+1}
		\node[label=above:{$0$}] at (19,11.1){};
		\node[label=above:{$\vdots$}] at (19,10){};
		\node[label=above:{$1$}] at (19,9.1){};
		\node[label=above:{$1$}] at (19,8.1){};
		\node[label=above:{$\vdots$}] at (19,7){};
		\node[label=above:{$1$}] at (19,6.1){};
		\node[label=above:{$1$}] at (19,5.1){};
		\node[label=above:{$1$}] at (19,4.1){};
		\node[label=above:{$1$}] at (19,3.1){};
		\node[label=above:{$1$}] at (19,2.1){};
		\node[label=above:{$1$}] at (19,1.1){};
		\node[label=above:{$0$}] at (19,0.1){};
		\node[label=above:{$0$}] at (19,-0.9){};
		\node[label=above:{$0$}] at (19,-1.9){};
		\node[label=above:{$\vdots$}] at (19,-3){};
		\node[label=above:{$0$}] at (19,-3.9){};
		\node[label=above:{$0$}] at (19,-4.9){};
		\node[label=above:{$0$}] at (19,-5.9){};
		
		% f^a_{1,2}
		\node[label=above:{$1$}] at (20,11.1){};
		\node[label=above:{$\vdots$}] at (20,10){};
		\node[label=above:{$1$}] at (20,9.1){};
		\node[label=above:{$1$}] at (20,8.1){};
		\node[label=above:{$\vdots$}] at (20,7){};
		\node[label=above:{$1$}] at (20,6.1){};
		\node[label=above:{$0$}] at (20,5.1){};
		\node[label=above:{$0$}] at (20,4.1){};
		\node[label=above:{$0$}] at (20,3.1){};
		\node[label=above:{$0$}] at (20,2.1){};
		\node[label=above:{$0$}] at (20,1.1){};
		\node[label=above:{$0$}] at (20,0.1){};
		\node[label=above:{$0$}] at (20,-0.9){};
		\node[label=above:{$0$}] at (20,-1.9){};
		\node[label=above:{$\vdots$}] at (20,-3){};
		\node[label=above:{$1$}] at (20,-3.9){};
		\node[label=above:{$1$}] at (20,-4.9){};
		\node[label=above:{$1$}] at (20,-5.9){};
		
		% ...
		\node[label=above:{$\cdots$}] at (21,11.1){};
		\node[label=above:{$\adots$}] at (21,10){};
		\node[label=above:{$\cdots$}] at (21,9.1){};
		\node[label=above:{$\cdots$}] at (21,8.1){};
		\node[label=above:{$\adots$}] at (21,7){};
		\node[label=above:{$\cdots$}] at (21,6.1){};
		\node[label=above:{$\cdots$}] at (21,5.1){};
		\node[label=above:{$\cdots$}] at (21,4.1){};
		\node[label=above:{$\cdots$}] at (21,3.1){};
		\node[label=above:{$\cdots$}] at (21,2.1){};
		\node[label=above:{$\cdots$}] at (21,1.1){};
		\node[label=above:{$\cdots$}] at (21,0.1){};
		\node[label=above:{$\cdots$}] at (21,-0.9){};
		\node[label=above:{$\cdots$}] at (21,-1.9){};
		\node[label=above:{$\adots$}] at (21,-3){};
		\node[label=above:{$\cdots$}] at (21,-3.9){};
		\node[label=above:{$\cdots$}] at (21,-4.9){};
		\node[label=above:{$\cdots$}] at (21,-5.9){};
		
		% f^a_{a,a+1}
		\node[label=above:{$0$}] at (22,11.1){};
		\node[label=above:{$\vdots$}] at (22,10){};
		\node[label=above:{$1$}] at (22,9.1){};
		\node[label=above:{$1$}] at (22,8.1){};
		\node[label=above:{$\vdots$}] at (22,7){};
		\node[label=above:{$1$}] at (22,6.1){};
		\node[label=above:{$1$}] at (22,5.1){};
		\node[label=above:{$1$}] at (22,4.1){};
		\node[label=above:{$0$}] at (22,3.1){};
		\node[label=above:{$0$}] at (22,2.1){};
		\node[label=above:{$0$}] at (22,1.1){};
		\node[label=above:{$0$}] at (22,0.1){};
		\node[label=above:{$0$}] at (22,-0.9){};
		\node[label=above:{$0$}] at (22,-1.9){};
		\node[label=above:{$\vdots$}] at (22,-3){};
		\node[label=above:{$1$}] at (22,-3.9){};
		\node[label=above:{$1$}] at (22,-4.9){};
		\node[label=above:{$1$}] at (22,-5.9){};
		
		% f^1
		\node[label=above:{$0$}] at (23,11.1){};
		\node[label=above:{$\vdots$}] at (23,10){};
		\node[label=above:{$0$}] at (23,9.1){};
		\node[label=above:{$0$}] at (23,8.1){};
		\node[label=above:{$\vdots$}] at (23,7){};
		\node[label=above:{$0$}] at (23,6.1){};
		\node[label=above:{$0$}] at (23,5.1){};
		\node[label=above:{$1$}] at (23,4.1){};
		\node[label=above:{$0$}] at (23,3.1){};
		\node[label=above:{$0$}] at (23,2.1){};
		\node[label=above:{$0$}] at (23,1.1){};
		\node[label=above:{$1$}] at (23,0.1){};
		\node[label=above:{$1$}] at (23,-0.9){};
		\node[label=above:{$1$}] at (23,-1.9){};
		\node[label=above:{$\vdots$}] at (23,-3){};
		\node[label=above:{$1$}] at (23,-3.9){};
		\node[label=above:{$1$}] at (23,-4.9){};
		\node[label=above:{$1$}] at (23,-5.9){};
		
		% f^2
		\node[label=above:{$0$}] at (24,11.1){};
		\node[label=above:{$\vdots$}] at (24,10){};
		\node[label=above:{$0$}] at (24,9.1){};
		\node[label=above:{$0$}] at (24,8.1){};
		\node[label=above:{$\vdots$}] at (24,7){};
		\node[label=above:{$0$}] at (24,6.1){};
		\node[label=above:{$0$}] at (24,5.1){};
		\node[label=above:{$1$}] at (24,4.1){};
		\node[label=above:{$1$}] at (24,3.1){};
		\node[label=above:{$1$}] at (24,2.1){};
		\node[label=above:{$1$}] at (24,1.1){};
		\node[label=above:{$0$}] at (24,0.1){};
		\node[label=above:{$0$}] at (24,-0.9){};
		\node[label=above:{$0$}] at (24,-1.9){};
		\node[label=above:{$\vdots$}] at (24,-3){};
		\node[label=above:{$1$}] at (24,-3.9){};
		\node[label=above:{$1$}] at (24,-4.9){};
		\node[label=above:{$1$}] at (24,-5.9){};
		
		% ...
		\node[label=above:{$\cdots$}] at (25,11.1){};
		\node[label=above:{$\adots$}] at (25,10){};
		\node[label=above:{$\cdots$}] at (25,9.1){};
		\node[label=above:{$\cdots$}] at (25,8.1){};
		\node[label=above:{$\adots$}] at (25,7){};
		\node[label=above:{$\cdots$}] at (25,6.1){};
		\node[label=above:{$\cdots$}] at (25,5.1){};
		\node[label=above:{$\cdots$}] at (25,4.1){};
		\node[label=above:{$\cdots$}] at (25,3.1){};
		\node[label=above:{$\cdots$}] at (25,2.1){};
		\node[label=above:{$\cdots$}] at (25,1.1){};
		\node[label=above:{$\cdots$}] at (25,0.1){};
		\node[label=above:{$\cdots$}] at (25,-0.9){};
		\node[label=above:{$\cdots$}] at (25,-1.9){};
		\node[label=above:{$\adots$}] at (25,-3){};
		\node[label=above:{$\cdots$}] at (25,-3.9){};
		\node[label=above:{$\cdots$}] at (25,-4.9){};
		\node[label=above:{$\cdots$}] at (25,-5.9){};
		
		% f^a
		\node[label=above:{$0$}] at (26,11.1){};
		\node[label=above:{$\vdots$}] at (26,10){};
		\node[label=above:{$0$}] at (26,9.1){};
		\node[label=above:{$0$}] at (26,8.1){};
		\node[label=above:{$\vdots$}] at (26,7){};
		\node[label=above:{$0$}] at (26,6.1){};
		\node[label=above:{$0$}] at (26,5.1){};
		\node[label=above:{$1$}] at (26,4.1){};
		\node[label=above:{$1$}] at (26,3.1){};
		\node[label=above:{$1$}] at (26,2.1){};
		\node[label=above:{$1$}] at (26,1.1){};
		\node[label=above:{$1$}] at (26,0.1){};
		\node[label=above:{$1$}] at (26,-0.9){};
		\node[label=above:{$1$}] at (26,-1.9){};
		\node[label=above:{$\vdots$}] at (26,-3){};
		\node[label=above:{$0$}] at (26,-3.9){};
		\node[label=above:{$0$}] at (26,-4.9){};
		\node[label=above:{$0$}] at (26,-5.9){};
		
		% e_enf
		\node[label=above:{$0$}] at (27,11.1){};
		\node[label=above:{$\vdots$}] at (27,10){};
		\node[label=above:{$0$}] at (27,9.1){};
		\node[label=above:{$0$}] at (27,8.1){};
		\node[label=above:{$\vdots$}] at (27,7){};
		\node[label=above:{$0$}] at (27,6.1){};
		\node[label=above:{$0$}] at (27,5.1){};
		\node[label=above:{$1$}] at (27,4.1){};
		\node[label=above:{$1$}] at (27,3.1){};
		\node[label=above:{$1$}] at (27,2.1){};
		\node[label=above:{$1$}] at (27,1.1){};
		\node[label=above:{$1$}] at (27,0.1){};
		\node[label=above:{$1$}] at (27,-0.9){};
		\node[label=above:{$1$}] at (27,-1.9){};
		\node[label=above:{$\vdots$}] at (27,-3){};
		\node[label=above:{$1$}] at (27,-3.9){};
		\node[label=above:{$1$}] at (27,-4.9){};
		\node[label=above:{$1$}] at (27,-5.9){};
		\end{tikzpicture}
		\caption{Sketch of the construction of \cref{thm:superexponential-lower-bound}.
			We only show three element examples~$e_u, e_v, e_w$ and we only show choice dimensions corresponding to sets~$F_1, F_2, F_3$.
			Here,~$F_1$ is a set containing elements~$u$ and~$w$, $F_2$ is a set containing element~$v$, and~$F_3$ is a set containing elements~$v$ and~$w$.}
		\label{fig-eth-bound-ell}
	\end{figure*}

	\emph{Construction:} 
	Let~$(\mathcal{U},\mathcal{F},a,x)$ be an instance of \textsc{Set Multicover} where each demand is equal to~$x$. 
	We construct an equivalent instance~$((E,\lambda),\ell, S)$ of \pMTES{} as follows.
	For a sketch of our construction we refer to \Cref{fig-eth-bound-ell}.
	
	First, we describe the training data set~$(E,\lambda)$.
	\begin{itemize}
		\item For each element~$u\in \mathcal{U}$, we add an \emph{element example}~$b_u$. 
		All these examples receive label~$\lpos$.
		
		\item Then, for each~$i\in[a+1]$, we add a \emph{forcing example}~$r_i$.
		To all such examples we assign label~$\lneg$.
		
		\item Afterwards, we add a \emph{validation example}~$b_\val$ and an \emph{enforcing example}~$b_\enf$.
		Both examples are~$\lpos$.
		
		\item Now, for each~$j\in[a]$ and each~$i\in[a+1]$, we add a \emph{choosing example}~$b^j_i$.
		All these examples receive label~$\lpos$.
		
		\item Next, for each~$j\in[a]$ and each~$1\le i < t\le a+1$, we add a \emph{verifying example}~$r^j_{i,t}$.
		To all such examples we assign label~$\lneg$.
		
		\item Finally, for each~$j\in[a]$, we add a \emph{test example}~$r^j$.
		All these examples receive label~$\lneg$.
	\end{itemize}

	Observe that we have $|\mathcal{U}|+(a+1)+2+a(a+1)+a\cdot a(a+1)/2+a=|\mathcal{U}|+a^3/2-a^2/2+3a+3$~examples.
	
	%We set~$\lambda(e_v)=\lpos$ for each vertex example,~$\lambda(e_i)=\lneg$ for each forcing example,~$\lambda(e_\val)=\lpos$, and~$\lambda(e_\enf)=\lpos$.
	
	To complete the description of the training data set, it remains to describe the number of dimensions~$d$ and the coordinates of each example in~$\mathds{R}^d$.
	
	We start with the description of the dimensions: 
	
	\begin{itemize}
		\item For each set~$F\in\mathcal{F}$ and each~$j\in[a]$, we introduce a dimension~$d^j_F$.
		We refer to these dimensions as the \emph{choice dimensions}.
		Moreover, we refer to the dimensions~$\{d^j_F: F \text{ is a set of } \mathcal{F}\}$ as the \emph{choice-$j$~dimensions}.
		
		\item Furthermore, for each~$i\in[a+1]$ we introduce a \emph{dummy dimension}~$d_i$.
	\end{itemize}

	Observe that we add exactly $a\cdot \mathcal{F}+a+1$~dimensions.
	
	Now, we describe the coordinates of the examples.
	By~$e[d_z]$ we denote the value of example~$e$ at dimension~$d_z$.
	Here,~$d_z$ can be a choice-, or a dummy dimension. 
	
	\begin{itemize}
		\item For each element example~$b_u$, each~$j\in[a]$, and each set~$F\in\mathcal{F}$, we set~$b_u[d^j_F]=1$ if~$u\in F$, otherwise, if~$u\notin F$, we set~$b_u[d^j_F]=0$.
		Furthermore, for each dummy dimension~$d_i$ we set~$b_u[d_i]=1$ if~$i\le a+1-x$, and otherwise, we set~$b_u[d_i]=0$.
		
		\item For each forcing example~$r_i$ and each choice dimension~$d^j_F$ we set~$r_i[d^j_F]=0$.
		For dummy dimension~$d_i$ we set~$r_i[d_i]=0$ and for each remaining dummy dimension~$d_j$, that is,~$j\in[a+1]\setminus\{i\}$, we set~$r_i[d_j]=1$.
		
		\item For the validation example~$b_\val$ we set~$b_\val[d^j_F]=0$ for each choice dimension and each~$j\in[a]$, and we set~$b_\val[d_i]=1$ for each dummy dimension~$d_i$.
		
		For the enforcing example~$b_\enf$ we set~$b_\enf[d^j_F]=1$ for each choice dimension.
		Then, we set~$b_\enf[d_1]=1$ for the first dummy dimension, and~$b_\enf[d_i]=0$ for each remaining dummy dimension~$d_i$ with~$i\in[2,a+1]$.
		
		\item For the choosing example~$b^j_i$ and dummy dimension~$d_i$ we set~$b^j_i[d_i]=0$ and for each remaining dummy dimension~$d_t$, that is,~$t\in[a+1]\setminus\{i\}$, we set~$b^j_i[d_t]=1$.
		Then, for each choice-$j$-dimension~$d^j_F$, we set~$b^j_i[d^j_F]=1$, and for each remaining choice dimension~$d^q_F$, that is,~$q\in[a]\setminus\{j\}$, we set~$b^j_i[d^q_F]=0$.
		
		\item For each verifying example~$r^j_{i,t}$ and each choice-$j$-dimension~$d^j_F$ we set~$r^j_{i,t}[d^j_F]=1$, and for each remaining choice dimension~$d^q_F$, that is,~$q\in[a]\setminus\{j\}$, we set~$r^j_{i,t}[d^q_F]=0$.
		Then, for dummy dimensions~$d_i$ and~$d_t$, we set~$r^j_{i,t}[d_i]=0=r^j_{i,t}[d_t]$, and for each remaining dummy dimension~$d_p$, that is,~$p\in[a+1]\setminus\{i,t\}$, we set~$r^j_{i,t}[d_p]=1$.
		
		\item For each test example~$r^j$, we set~$r^j[d_1]=1$ for the first dummy dimension, and~$r^f[d_i]=0$ for each remaining dummy dimension~$d_i$ with~$i\in[2,a+1]$.
		Then, for each choice-$j$-dimension~$d^j_F$ we set~$r^j[d^j_F]=0$, and for each remaining choice dimension~$d^q_F$, that is,~$q\in[a]\setminus\{j\}$, we set~$r^j[d^q_F]=1$.
	\end{itemize}

	Finally, we set~$\ell\coloneqq 2a+1$ and~$S\coloneqq 2a+1$.
	In other words, the tree ensemble contains exactly $2a+1$~decision trees and all these trees together have $2a+1$~inner nodes.
	
	\emph{Intuition:}
	Since the validation example~$b_\val$ and each forcing example~$r_i$ are classified differently, the ensemble has to contain a cut separating~$b_\val$ and~$r_i$. 
	Since the unique dimension in which~$b_\val$ and~$r_i$ have different value is~$d_i$, the ensemble contains a cut in each dummy dimension~$d_i$.
	
	Furthermore, observe that for each~$j\in[a]$ the forcing example~$r_i$ and the choosing example~$b^j_i$ are classified differently and that they only differ in the choice-$j$~dimensions.
	Hence, for each~$j\in[a]$ the ensemble contains a cut in the choice-$j$~dimensions.
	Since~$S=2a+1$, we thus obtain a characterization of the cuts of the ensemble.
	
	Next, we use this characterization to show that each tree in the ensemble has exactly one inner node.
	Hence, $a+1$~trees of the ensemble perform cuts in the dummy dimensions and $a$~trees of the ensemble perform cuts in the choice dimensions.
	The cuts in the choice dimensions then corresponds to a set of $a$~many sets~$\mathcal{F}'$ of~$\mathcal{F}$.
	The sets of~$\mathcal{F}'$ then correspond to a solution of~$(\mathcal{U},\mathcal{F},a,x)$.
	
	\emph{Correctness:} 
	We show that~$(\mathcal{U},\mathcal{F},a,x)$ is a yes-instance of \textsc{Set Multicover} if and only if~$(E,\lambda,\ell, S)$ is a yes-instance of \pMTES.
	
	$(\Rightarrow)$
	Let~$\mathcal{F}'$ be a solution for~$(\mathcal{U},\mathcal{F},a,x)$, that is, for each element~$u\in\mathcal{U}$ there exists at least $b$~sets in~$\mathcal{F}'$ which contain~$u$. 
	Let~$F_1,\ldots, F_a$ be an arbitrary but fixed ordering of~$\mathcal{F}'$.
	We show that there exists a tree ensemble~$\mathcal{T}$ that classifies~$(E,\lambda)$.
	For each set~$F_c\in\mathcal{F}'$ with~$c\in[a]$, let~$d^c_{F_c}$ be the choice-$c$ dimension corresponding to~$F_c$.
	For each~$c\in [a]$ we add a tree~$T_c$ to~$\mathcal{T}$ cutting the choice-$c$ dimension~$d^c_{F_c}$ such that all examples~$e$ with~$e[d^c_{F_c}]=1$ are assigned label~$\lpos$ and all remaining examples~$e$ with~$e[d^c_{F_c}]=0$ are assigned label~$\lneg$.
	We denote these trees as the \emph{choice trees}.
	
	Next, for each of the $a+1$~dummy dimensions~$d_i$, we add a tree~$T_i$ to~$\mathcal{T}$ such that all examples~$e$ with~$e[d_i]=0$ are assigned label~$\lneg$ and all remaining examples~$e$ with~$e[d_i]=1$ are assigned label~$\lpos$.
	We denote these trees as the \emph{dummy trees}. 
	
	Observe that~$\mathcal{T}$ consists of exactly $\ell=2a+1$~trees with one internal node each.
	We now verify that~$\mathcal{T}$ classifies~$(E,\lambda)$ correctly.
	We distinguish all different types of examples:
	
	\begin{itemize}
		\item We show that~$\mathcal{T}(b_u)=\lpos$ for each element example.
		In the dummy trees,~$b_u$ is assigned label~$\lpos$ exactly $a+1-x$~times.
		Since~$\mathcal{F}'$ is a solution, at least $x$~choice trees assign label~$\lpos$ to~$b_u$.
		Thus,~$b_u$ is classified as~$\lpos$.
		
		\item We show that~$\mathcal{T}(r_i)=\lneg$ for each forcing example.
		In the dummy trees~$r_i$ is assigned label~$\lneg$ exactly once.
		In each of the $a$~many choice trees,~$r_i$ is assigned label~$\lneg$.
		Thus,~$r_i$ is classified as~$\lneg$.
		
		\item We show that~$\mathcal{T}(b_\val)=\lpos$.
		Since in each of the $a+1$~many dummy trees~$b_\val$ is assigned label~$\lpos$, the statement follows.
		
		\item We show that~$\mathcal{T}(b_\enf)=\lpos$.
		Since in the dummy tree cutting dimension~$d_1$ and in each of the $a$~many choice trees~$b_\enf$ is assigned label~$\lpos$, the statement follows.
		
		\item We show that~$\mathcal{T}(b^j_i)=\lpos$ for each choosing example.
		In the dummy trees except the one cutting~$d_i$, example~$b^j_i$ receives label~$\lpos$.
		Furthermore, in the choice tree cutting a choice-$j$~dimensions,~$b^j_i$ also is classified as~$\lpos$.
		Thus,~$b^j_i$ is classified as~$\lpos$.
		
		\item We show that~$\mathcal{T}(r^j_{i,t})=\lneg$ for each verifying example.
		In exactly two dummy trees, namely the ones cutting dummy dimensions~$d_i$ and~$d_t$, example~$r^j_{i,t}$ is assigned label~$\lneg$.
		Furthermore, in all choice trees except the one doing a cut in a choice-$j$~dimension,~$r^j_{i,t}$ is assigned label~$\lneg$.
		Thus, the statement follows.
		
		\item We show that~$\mathcal{T}(r^j)=\lneg$ for each test example.
		In each dummy tree except the one cutting~$d_1$ example~$r^j$ is assigned label~$\lneg$.
		Also in the choice tree doing a cut in the choice-$j$~dimensions,~$r^j$ is assigned label~$\lneg$.
		Thus,~$\lambda(r^j)=\lneg$.
	\end{itemize}
	
	Hence,~$\mathcal{T}$ classifies~$(E,\lambda)$ correctly.
	
	$(\Leftarrow)$ 
	Conversely, suppose that there exists a tree ensemble~$\mathcal{T}$ which classifies~$(E,\lambda)$.
	In a first step, we show that~$\mathcal{T}$ contains exactly one cut in each dummy dimension and exactly one cut in the choice-$j$~dimensions for each~$j\in[a]$.
	Since~$S=2a+1$ this gives us a characterization of all cuts done by~$\mathcal{T}$.
	In a second step, we show that each tree of~$\mathcal{T}$ has exactly one inner node.
	Hence, exactly $a$~trees of~$\mathcal{T}$ cut some choice dimension.
	In a third step, we show that each element example~$b_v$ is classified at least $x$~times as~$\lpos$ by trees of~$\mathcal{T}$ cutting a choice dimension.
	In the final step, we construct a solution~$\mathcal{F}'$ for~$(\mathcal{U},\mathcal{F},a,x)$ based on the cuts in the choice dimensions.

	%In a first step, we show that for each dummy dimension~$\mathcal{T}$ contains a tree cutting this dummy dimension and assigning all examples with value~$0$ symbol~$\lneg$ and each example with value~$1$ symbol~$\lpos$.
	%In a second step, we show that all remaining trees do cuts in choice dimensions such that all examples with value~$0$ are assigned symbol~$\lneg$ and each example with value~$1$  is assigned symbol~$\lpos$. 
	%In the final step, we construct a $b$-folding function based on the cuts in the choice dimensions.
	
	\emph{Step 1:}
	Recall that the validation example~$v_\val$ has label~$\lpos$ and that each forcing example~$r_i$ has label~$\lneg$. 
	Observe that for each~$i\in[a+1]$ the forcing example~$r_i$ and the validation example~$b_\val$ differ in exactly one dimension: the dummy dimension~$d_i$, that is,~$r_i[d_i]=0$ and~$b_\val[d_i]=1$.
	Thus, for each~$i\in[a+1]$, the ensemble~$\mathcal{T}$ has to contain a cut in dimension~$d_i$.
	
	Also recall that for each~$j\in[a]$, each choosing example has label~$\lpos$.
	Observe that for each~$i\in[a+1]$ and for each~$j\in[a]$, the forcing example~$r_i$ and the choosing example~$b^j_i$ only differ in the choice-$j$-dimensions: in these,~$r_i$ has value~$0$ and~$b^j_i$ has value~$1$.
	Hence, in~$\mathcal{T}$ there is at least one cut in some choice-$j$~dimension to distinguish~$r_i$ and~$b^j_i$.
	Since there are $a$~many choices for~$j$, since~$S=2a+1$, and since~$\mathcal{T}$ has $a+1$~cuts in dummy dimensions, we conclude that exactly one cut in~$\mathcal{T}$ is done in each dummy dimension and exactly one cut in~$\mathcal{T}$ is done in the choice-$j$~dimensions for any~$j\in[a]$.
	
	\emph{Step 2:}
	Next, we show that each tree in~$\mathcal{T}$ consists of exactly one inner node.
	Assume towards a contradiction that some tree~$T\in\mathcal{T}$ has at least two~inner nodes.
	We now show that in this case there exist a pair of examples with different labels which cannot be distinguished by~$\mathcal{T}$, which implies that~$\mathcal{T}$ does not classify~$(E, \lambda)$.
	By~$T_\lef$ and~$T_\rig$ we denote the left and right subtree of~$T$, respectively.
	Next, we distinguish whether in the root of~$T$ a dummy dimension or a choice dimension is cut.
	
	\emph{Case 1:}
	Assume that the root of~$T$ cuts the dummy dimension~$d_i$. 
	Without loss of generality, we assume that~$T_\lef$ contains all examples~$e$ with~$d_i[e]=0$ and that~$T_\rig$ contains all examples~$e$ with~$d_i[e]=1$.
	
	First, assume that another inner node of~$T$ cuts the dummy dimension~$d_t$.
	Recall that we have~$t\ne i$ since each dummy dimension is cut exactly once by~$\mathcal{T}$.
	Observe that~$T_\lef$ contains the choosing example~$b^1_i$ and the verifying example~$r^1_{i,t}$ (in the following we assume that~$i<t$ since otherwise all arguments can be done with example~$r^1_{t,i}$, since only in this case the example exists), and that~$T_\rig$ contains the validation example~$b_\val$ and the forcing example~$r_t$.
	Note that~$b^1_t$ and~$r^1_{i,t}$, and also~$b_\val$ and~$r_t$ only differ in the dummy dimension~$d_t$.
	Since in~$\mathcal{T}$ dummy dimension~$d_t$ is cut exactly once and since this cut is done in some inner node of~$T$, this cut is done either in~$T_\lef$ or in~$T_\rig$.
	In any case, in the remaining subtree~$T_\lef$ or~$T_\rig$ there are then at least two examples which cannot be distinguished by~$\mathcal{T}$, a contradiction.
	
	Second, assume that another inner node of~$T$ cuts a choice-$j$~dimension.
	Observe that~$T_\lef$ contains the forcing example~$r_i$ and the choosing example~$b^j_i$, and that~$T_\rig$ contains the forcing example~$r_t$ and the choosing example~$b^j_t$ where~$t$ is any number in~$[a+1]\setminus\{i\}$.
	Note that~$r_i$ and~$b^j_i$, and also~$r_t$ and~$b^j_t$ only differ in the choice-$j$~dimensions.
	Since in~$\mathcal{T}$ there is exactly one cut in the choice-$j$~dimensions and since this cut is done in some inner node of~$T$, this cut is done either in~$T_\lef$ or~$T_\rig$.
	In any case, in the remaining subtree~$T_\lef$ or~$T_\rig$ there are then at least two examples which cannot be distinguished by~$\mathcal{T}$, a contradiction.
	
	\emph{Case 2:}
	Assume that the root of~$T$ cuts some choice-$j$~dimension~$d^j_F$. 
	Without loss of generality, we assume that~$T_\lef$ contains all examples~$e$ with~$d^j_F[e]=0$ and that~$T_\rig$ contains all examples~$e$ with~$d^j_F[e]=1$.
	
	First, assume that another inner node of~$T$ cuts the dummy dimension~$d_t$.
	Observe that~$T_\lef$ contains the choosing example~$b^z_i$ and the verifying example~$r^z_{i,t}$ (similar to Case 1, in the following we assume that~$i<t$), and that~$T_\rig$ contains the choosing example~$b^j_i$ and the verifying example~$r^j_{i,t}$.
	Here,~$z\in[a]\setminus\{j\}$ and~$i\in[a+1]\setminus\{t\}$.
	Note that~$b^z_i$ and~$r^z_{i,t}$, and also~$b^j_i$ and~$r^j_{i,t}$ only differ in the dummy dimension~$d_t$.
	Since in~$\mathcal{T}$ dummy dimension~$d_t$ is cut exactly once and since this cut is done in some inner node of~$T$, this cut is done either in~$T_\lef$ or in~$T_\rig$.
	In any case, in the remaining subtree~$T_\lef$ or~$T_\rig$ there are then at least two examples which cannot be distinguished by~$\mathcal{T}$, a contradiction.
	
	Second, assume that another inner node of~$T$ cuts a choice-$z$~dimension.
	Recall that we have~$z\ne j$ since exactly one cut is done in all choice-$j$~dimensions.
	Observe that~$T_\lef$ contains the forcing example~$r_i$ and the choosing example~$b^z_i$ where~$i\in[a+1]$, and that~$T_\rig$ contains the enforcing example~$b_\enf$ and the test example~$r^z$.
	Note that~$r_i$ and~$b^z_i$, and also~$b_\enf$ and~$r^z$ only differ in the choice-$z$~dimensions.
	Since in~$\mathcal{T}$ there is exactly one cut in the choice-$z$~dimensions and since this cut is done in some inner node of~$T$, this cut is done either in~$T_\lef$ or~$T_\rig$.
	In any case, in the remaining subtree~$T_\lef$ or~$T_\rig$ there are then at least two examples which cannot be distinguished by~$\mathcal{T}$, a contradiction.
	
	Hence, we have verified that each tree in~$\mathcal{T}$ has exactly one inner node.
	
	%Let~$\mathcal{T}_{a+1}\coloneqq \{T_i:i\in[a+1]\}$ and let~$\mathcal{T'}\coloneqq \mathcal{T}\setminus\mathcal{T}_{a+1}$.
	%Note that~$\mathcal{T}_{a+1}$ classifies each forcing example~$e_i$ once as~$\lneg$, namely with~$T_i$, and $a$~times as~$\lpos$.
	%Since~$e_i$ has symbol~$\lneg$, each tree in~$\mathcal{T}'$ has to assign symbol~$\lneg$ to each forcing example~$e_i$.
	%Furthermore, observe that~$\mathcal{T}_{a+1}$ classifies the enforcing example~$e_\enf$ once as~$\lpos$, namely with~$T_1$, and $a$~times as~$\lneg$.
	%Since~$e_\enf$ has symbol~$\lpos$, each tree in~$\mathcal{T}'$ has to assign symbol~$\lpos$ to~$e_\enf$.
	
	%Now, observe that only in forcing dimensions the enforcing example~$e_\enf$ has a different value than each forcing example.
	%More precisely,~$e_\enf[d_I]=1$ and~$e_i[d_I]=0$ for each choice dimension~$d_i$ and each forcing example~$e_i$.
	%Thus, each tree in~$\mathcal{T}'$ is doing a cut in a choice dimension~$d_I$ such that all examples~$e$ with~$e[d_I]=0$ are assigned symbol~$\lneg$ and all examples~$e$ with~$e[d_I]=1$ are assigned symbol~$\lpos$.
	
	\emph{Step 3:}
	Since each tree in the ensemble~$\mathcal{T}$ has exactly one inner node, we conclude that for each~$i\in[a+1]$ the ensemble~$\mathcal{T}$ contains a tree which cuts dummy dimension~$d_i$.
	Observe that since the validation example~$b_\val$ has label~$\lpos$, since the forcing example~$r_i$ has label~$\lneg$, and since~$b_\val$ and~$r_i$ only differ in dummy dimension~$d_i$, we conclude that by this tree all examples~$e$ with~$e[d_i]=0$ are classified as~$\lneg$ and all examples~$e$ with~$e[d_i]=1$ are classified as~$\lpos$.
	
	Since each tree in the ensemble~$\mathcal{T}$ has exactly one inner node, we also conclude that for each~$j\in[a]$ the ensemble~$\mathcal{T}$ contains a tree which cuts exactly one choice-$j$-dimension.
	Observe that the forcing example~$r_i$ has label~$\lneg$, that the choosing example~$b^j_i$ has label~$\lpos$, and that~$r_i$ and~$b^j_i$ only differ in the choice-$j$~dimensions.
	Here,~$i$ is some arbitrary but fixed integer of~$[a+1]$.
	Hence, the tree of the ensemble cutting a choice-$j$~dimensions, say~$d^j_F$, does the following: all examples~$e$ with~$e[d^j_F]=0$ are classified as~$\lneg$ and all examples~$e$ with~$e[d^j_F]=1$ are classified as~$\lpos$.
	
	Now, consider the element example~$b_u$ for some~$u\in \mathcal{U}$.
	Observe that~$b_u$ is classified $a+1-x$~times as~$\lpos$ by trees of~$\mathcal{T}$ doing cuts in the dummy dimensions.
	Since~$\mathcal{T}$ classifies~$(E,\lambda)$, we conclude that~$b_u$ is classified at least $x$~times as~$\lpos$ by trees of~$\mathcal{T}$ cutting a choice dimension.
	
	\emph{Step 4:}
	Recall that each choice dimension~$d^j_F$ for each~$j\in[a]$ corresponds to a set~$F\in\mathcal{F}$.
	Let~$\mathcal{T}'=\{T^j:j\in [a]\}$ such that~$T^j$ cuts the choice dimension~$d^j_{F_j}$ for some~$j\in[a]$.
	%Here,~$F_1, \ldots, F_a$ is an arbitrary but fixed ordering of the chosen sets.
	Note that in~$T^j$ all examples corresponding to~$F_j$ are assigned label~$\lpos$, and all remaining elements in~$\mathcal{U}\setminus F_j$ receive label~$\lneg$.
	According to Step~$3$, each example~$b_u$ for each~$u\in\mathcal{U}$ is classified at least $x$~times as~$\lpos$ by~$\mathcal{T}'$.
	Thus,~$\mathcal{F}'\coloneqq \{F_j:j\in[a]\}$ is a solution for~$(\mathcal{U},\mathcal{F},a,x)$.

	%Now, let~$\fol$ be the function mapping vertices to subsets of~$[a]$ (the set of colors assigned to the vertices) such that~$\fol(v)\coloneqq \{i:v\in I_i\}$.
	%Note that, according to Step~$3$,~$\fol(v)$ contains at least $x$~colors for each~$v\in V(G)$ since~$\mathcal{T}'$ assigns symbol~$\lpos$ at least $x$~times to the element example~$b_v$ corresponding to~$v$.
	%Thus, by shrinking~$\fol(v)$ arbitrarily to sets of $x$~colors for each element~$v\in V(G)$, we have verified that~$\fol$ is $x$-folding for~$(G,a,x)$.
	
	\emph{Lower Bound:}
	%Since we add $n+a+3$~examples and $3^{n/3}+a+1$~dimensions, the construction needs $\Oh(n\cdot 3^{n/3})$~time.
	Recall that the construction adds $|\mathcal{U}|+a^3/2-a^2/2+3a+3$~examples and $a\cdot |\mathcal{U}|+a+1$~dimensions and is clearly implementable in polynomial time.
Thus, if \pMTES{} has an algorithm with running time~$f(\ell)\cdot 2^{o(\log \ell)\cdot n}\cdot \poly$, then we obtain an algorithm for \textsc{Set Multicover} with running time~$f(x) \cdot 2^{o(\log x)\cdot |\mathcal{U}|}\cdot \poly$, a contradiction to the ETH.	
	%Now, if \pMTES{} has an algorithm with running time~$f(\ell)\cdot 2^{o(\log \ell)\cdot n}$ and since we have~$\ell=2a+1$, this implies an algorithm with running time~$f(a)\cdot 2^{o(\log a)\cdot n}+n^4\cdot 2^n$ for \abColor.
	%Since~$a\in\Theta(x^2 \log x)$, this then implies an algorithm running in $f(x)\cdot 2^{o(\log x^3)\cdot n}=f(x)\cdot 2^{o(\log x)\cdot n}$~time, a contradiction to the ETH~\cite{bonamy_tight_2019}.
\end{proof}

Now, \cref{thm:superexponential-lower-bound} implies that the running time~$(\ell+1)^n\cdot\poly$ of the algorithm in \cref{thm:exptime-algo} cannot be significantly improved, unless the ETH is wrong.

Our proof of \cref{thm:superexponential-lower-bound} also implies hardness for the larger parameter~$S$ and that this hardness holds even if $D=2$, that is, each feature is binary.

\begin{corollary}\label[corollary]{cor:superexponential-lower-bound}
	Solving \pMTESlong\ on instances with binary features in $f(S) \cdot  2^{o(\log S)\cdot n}$~time would contradict the \ethlong.
\end{corollary}

Our proof of \cref{thm:superexponential-lower-bound} also implies hardness for the minimax optimization goal.
For this result the proof is simpler, since no argument is needed that each tree in the ensemble has exactly one inner node.

\begin{corollary}\label[corollary]{cor:superexponential-lower-bound-mmaxtes}
	Solving \pMMTESlong\ on instances with binary features and~$s=1$ in $f(\ell) \cdot  2^{o(\log \ell)\cdot n}$~time would contradict the ETH.
\end{corollary}

\section{Extensions}\label{sec:extensions}

We now discuss how our algorithmic results can be adapted to more general settings. 
First, we show how to handle more than two classes.
Second we allow up to $t$~misclassifications for the resulting tree in the training data set.

\subsection{Non-Binary Classification}
Recall that~$\Sigma$ is the set of classes. 
To adapt the witness-tree algorithm (see \Cref{thm:witness-tree-algo}) for more than two classes, we do the following:
In the initialization of the ensemble, for each tree~$T\in\mathcal{C}$ we choose an arbitrary example~$e$ and try all $|\Sigma|$~possibilities for whether~$e$ is correctly classified (as~$\lambda(e)$) or not, that is, we test all $|\Sigma|-1$~possibilities for the class of the leaf in~$T$ which contains~$e$.
In total, these are $|\Sigma|^{\ell}$~possibilities.
Afterwards, branching works as for 2~classes: 
First, we select an arbitrary dirty example~$e$.
Second, we branch into all possibilities for a tree~$T$ of~$\mathcal{C}$ where~$e$ is currently misclassified and not a witness. 
Finally, for each such tree~$T$ we branch into all important one-step refinements of~$T$ introducing~$e$ as the new witness.
Thus, the running time for branching is independent of the size of~$\Sigma$.
Hence, we derive the following.
Interestingly, the running time for \pDTSlong\ is independent of~$\Sigma$.

\begin{proposition}\label{prop:witness-tree-algo}
	For any set~$\Sigma$ of classes, \pMTESlong\ can be solved in $\Oh(|\Sigma|^{\ell}\cdot (2 \delta D S)^{S} \cdot S \ell dn)$ time and \pMMTESlong\ in $\Oh(|\Sigma|^{\ell} \cdot (\delta D \ell s)^{s\ell} \cdot s \ell^2 dn)$ time.
\end{proposition}

\begin{corollary}
	For any set~$\Sigma$ of classes, \pDTSlong\ can be solved in $\Oh((\delta D s)^s \cdot sdn)$~time.
\end{corollary}

Next, we adapt the exponential-time algorithms for a small number of examples of \Cref{sec-algos-for-n} for~$|\Sigma|=2$ to arbitrary but fixed~$\Sigma$.
Recall that in \Cref{lem-blue-correct} we first presented a dynamic-programming algorithm to compute for any example subset~$E'\subseteq E$ the size of a smallest decision tree classifying all examples in~$E'$ as~$\lpos$. 
Second, we used \Cref{lem-blue-correct} to show that \pDTSlong\ can be solved in $2^n \cdot \poly$~time (\Cref{cor:single-exponential-decision-trees}).
Finally, in \Cref{thm:exptime-algo} we used \Cref{lem-blue-correct} to solve \pMTESlong\ in $(\ell + 1)^n \cdot \poly$ time.

In \Cref{lem-blue-correct}, the dynamic-programming table~$Q$ has one entry~$E^b$ and one entry~$E^r$ such that~$Q[E^b,E^r]$ stores the size of a smallest decision tree on~$E^b\cup E^r$ where exactly the examples in~$E^b$ receive label~$\lpos$ and the examples in~$E^r$ receive label~$\lneg$.
Now, in our adaption to general~$\Sigma$, table~$Q$ has $|\Sigma|$~entries, that is, one entry~$E^i$ for each~$i\in\Sigma$,
Furthermore,~$Q[E^1, \ldots, E^{|\Sigma|}]$ stores the size of a smallest decision tree on the example set~$E^1\cup\ldots\cup E^{|\Sigma|}$ where exactly the examples in~$E^i$ receive class label~$i$.
As before, we fill~$Q$ for increasing number of examples.
Initially, for each set~$E^i\subseteq E$ of examples, we set~$Q[\emptyset,\ldots, \emptyset, E^i,\emptyset, \ldots, \emptyset]=0$.
The recurrence is  similar to before: 
We iterate over all possible dimensions and each possible threshold.
Since each entry of~$Q$ corresponds to a partition of~$E$ into $(|\Sigma|+1)$~parts ($|\Sigma|$~parts for the entries of~$Q$ and one part for the unused examples), $Q$ has $(|\Sigma|+1)^n$~entries.
Since each entry can be computed in polynomial time, we obtain the following:

\begin{lemma}
\label{lem-blue-correct2}
	Given a training data set~$(E,\lambda)$, we can compute in $\Oh((|\Sigma|+1)^n \cdot Ddn)$~time for each partition~$(E^1,\ldots, E^{|\Sigma|})$ of~$E$, the size of a smallest decision tree~$T$ such that exactly the examples in~$E^i$ receive class label~$i$.
\end{lemma}

\Cref{lem-blue-correct2} directly implies a $(|\Sigma|+1)^n \cdot \poly$~time algorithm for \pDTSlong{}.
Similarly to \Cref{cor:single-exponential-decision-trees}, we show that this running time can be significantly improved.
Interestingly, the resulting running time is independent of~$\Sigma$.
Furthermore, due to \Cref{thm:witness-tree-algo-tight}, an algorithm with running time~$(2-\varepsilon)^n$ for any~$\varepsilon>0$ is not possible unless the Set Cover Conjecture is wrong.

\begin{theorem}\label{cor:single-exponential-decision-trees2}
	For any set~$\Sigma$ of classes, \pDTSlong\ can be solved in $\Oh(2^n \cdot Ddn)$~time.
\end{theorem}
\begin{proof}
  We use the algorithm from \Cref{lem-blue-correct2}, but now we restrict the table such that it contains only those entries $Q[E^1,\ldots, E^{|\Sigma|}]$ where $E^i$ is a subset of the examples in class~$i$ of the input.
  The initialization, the recurrence, and the solution can be be computed analogously.
Similar to  \Cref{cor:single-exponential-decision-trees}, the algorithm remains correct since the solution trees for \pDTS\ do not have misclassifications.
  As for the running time, observe that the number of entries is~$\Oh(2^{a_1} \cdot 2^{a_2}\cdot\ldots\cdot 2^{a_{|\Sigma|}})$, where $a_i$ is the number of examples of class~$i$ in the input.
  This yields the claimed bound because $a_1 + a_2 + \ldots + a_{|\Sigma|} = n$.
\end{proof}

The dynamic-programming recurrence of \Cref{thm:exptime-algo} works analogously for more than two classes.
The only distinction is the running-time bound; now we have a factor of $(|\Sigma|+1)^n$ instead of $3^n$ for the preprocessing.
Thus, we obtain the following.

\begin{theorem}\label{prop:exptime-algo}
	For any set~$\Sigma$ of classes and for~$\ell>1$, we can solve \pMTESlong\ in $\Oh(((|\Sigma|+1)\cdot(\ell + 1))^n \cdot Ddn)$ time.
\end{theorem}

\subsection{Error Minimization}
Now, we focus on generalizations where up to~$t$ misclassifications are allowed.  
Note that \cite{GZ24} studied the special case of our generalization where the ensemble has size~1, that is, $\ell=1$.
Formally, we study the following problem:

\probdef{\pMTESOlong\ (\pMTESO)}
{A training data set $(E, \lambda)$, a number~$\ell$ of trees, a size bound~$S$, and a bound~$t$ on the number of misclassifications.}
{Is there a tree ensemble of overall size at most $S$ that classifies $(E, \lambda)$ except at most $t$~examples?}

We denote the special case of~$\ell=1$ as \pDTSOlong\ (\pDTSO).
%The special case of \pMTESO\ with~$\ell = 1$ is known as \pDTSlong\ (\pDTSO)~\cite{GZ24}.
In the problem variant \pMMTESOlong\ (\pMMTES), instead of $S$, we are given an integer $s$ and we ask whether there is a tree ensemble that classifies~$(E, \lambda)$ except at most $t$~examples, and that contains exactly~$\ell$ trees, each of which has size at most~$s$.

First, we adapt \Cref{thm:witness-tree-algo} to \pMTESO:
The initialization can be done analogously.
For branching, it is now not sufficient to choose a dirty example~$e$ and branch on all important one-step refinements since in the final ensemble, $e$ might be misclassified.
Thus, we need to pick an arbitrary set of $(t+1)$~dirty examples and branch for each of them in all possible important one-step refinements.
Thus, we obtain the following.

\begin{proposition}\label{prop:witness-tree-algo}
	\pMTESOlong\ can be solved in $\Oh((4(t+1) \delta D S)^{S} \cdot S \ell dn)$ time and \pMMTESOlong\ in $\Oh(2^{\ell} \cdot ((t+1)\delta D \ell s)^{s\ell} \cdot s \ell^2 dn)$ time.
\end{proposition}

\begin{corollary}
\label{cor-dtso}
	For any~$\Sigma$, \pDTSOlong\ can be solved in $\Oh(((t+1)\delta D s)^s \cdot sdn)$~time.
\end{corollary}

Very recently, it was shown that \pDTSO\ is FPT for~$t+s+\delta$~\cite[Theorem~2]{GZ24} and that \pDTSO\ is W[1]-hard for~$s$, even if~$\delta\le 3$~\cite[Theorem~1]{GZ24}.
The FPT-algorithm for~$t+s+\delta$ requires some enumeration steps which are infeasible in practice.
Our algorithm behind \Cref{cor-dtso} additionally requires parameter~$D$, but is a branch-and-bound algorithm which is more easily amenable for heuristic improvements.

Next, we focus on the algorithms with exponential dependence on the number~$n$ of examples.
Observe that \Cref{lem-blue-correct} and \Cref{cor:single-exponential-decision-trees} are also valid for \pDTSOlong\ :
The initialization and the updates of table~$Q$ are performed analogously to \Cref{lem-blue-correct}. 
In \Cref{lem-blue-correct}, where the aim is to not have any misclassifications, the result is the unique entry of~$Q$ where each each $\lpos$~example is contained in~$E_b$ and each $\lneg$~example is contained in~$E_r$.
Now, since we allow up to $t$~misclassifications, we need to consider the minimal value of~$Q$ where at most $t$~examples are misclassified.
Formally, the result of the dynamic-programming algorithm is
$$\min_{(E_1^*,E_2^*,\ldots, E_{|\Sigma|}^*)} Q(E_1^*,E_2^*,\ldots, E_{|\Sigma|}^*)$$ 
where~$(E_1^*,E_2^*,\ldots, E_{|\Sigma|}^*)$ is a partition of~$E$ and all examples in~$E_i^*$ receive label~$i$ such that for at most $t$~elements~$e$ the label of~$e$ with respect to~$(E_1^*,E_2^*,\ldots, E_{|\Sigma|}^*)$ is different from the class label of~$e$.
Thus, we obtain the following:

\begin{corollary}\label{cor:single-exponential-decision-trees-opt}
	\pDTSOlong\ can be solved in $\Oh(3^n \cdot Ddn)$~time.
\end{corollary}

Note that this dynamic programming approach can also be used to give other sets of Pareto-optimal trees for the trade-off between size and essentially any type of efficiently computable classification error, for example~precision, recall, or~$F_1$-score.
Formally, let~$\alpha$ be a minimal threshold one aims to achieve for one of these scores.
Then, the result of the dynamic-programming algorithm can be found in the entry
$$\min_{(E_1^*,E_2^*,\ldots, E_{|\Sigma|}^*)} Q(E_1^*,E_2^*,\ldots, E_{|\Sigma|}^*)$$
where~$(E_1^*,E_2^*,\ldots, E_{|\Sigma|}^*)$ is a partition of~$E$ and all examples in~$E_i^*$ receive label~$i$ such that this partition achieves score at least~$\alpha$.

\Cref{thm:exptime-algo} can be adapted for \pMTESOlong\ as follows:
The initialization and the updates are computed similarly.
The result of the dynamic-programming algorithm is the minimal value of each entry~$R[c,\ell]$ where~$c$ is a vector with at most $t$~entries which are smaller than~$\lceil\ell/2\rceil$.
Thus, we obtain the following.

\begin{proposition}\label{prop:exptime-algo-opt}
	For~$\ell>1$, one can solve \pMTESOlong\ in $\Oh((\ell + 1)^n \cdot Ddn)$ time.
\end{proposition}

Furthermore, we note that the generalizations to more than two classes and up to $t$~misclassifications can be combined. Our adaptions for the individual generalizations can be straightforwardly combined to yield algorithms for these problems.

\subsection{Enumeration}

Finally, we remark that our algorithms to find a minimum tree ensemble which classifies~$(E,\lambda)$ correctly (\Cref{thm:witness-tree-algo,cor:single-exponential-decision-trees,thm:exptime-algo}) and our adaptions for more than two classes (\Cref{prop:witness-tree-algo,cor:single-exponential-decision-trees2,prop:exptime-algo}) and for up to $t$~errors (\Cref{prop:witness-tree-algo,cor:single-exponential-decision-trees-opt,prop:exptime-algo-opt}) not only can be used to find \emph{one} solution, instead they can also be used to enumerate \emph{all} solutions.

For the witness-tree algorithm and its adaptions we do not terminate if a solution is detected, instead we continue to search the entire search space until all solutions are detected.
Furthermore, we let the algorithm run for each possible tree size up to~$s$.
For the dynamic-programming approach, we again compute all entries of the table and then output all trees corresponding to entries where~$(E,\lambda)$ is classified (except at most $t$~examples, if we allow up to $t$~errors) which have size at most~$s$.

\section{Outlook}\label{sec:outlook}
We conclude by mentioning a few avenues for possible future research.
Our results are theoretical and need to be practically verified.
The natural next step is to empirically evaluate the proposed algorithms with respect to their running time and memory consumption.
As the witness-tree algorithm applies not only to ensembles but also to plain decision trees, it is interesting to see a comparison to the state-of-the-art in computing minimum-size decision trees.
After such evaluations, it is interesting to explore the trade offs that minimum-size ensembles offer on practical benchmark data sets between their sizes and accuracy on validation data.
In particular, it would be interesting to see a comparison of the random-forest approach~\cite{breiman_random_2001} of computing trees on subsets of the training data to obtain an ensemble of decision trees versus globally optimizing the size of an ensemble.

On the theory-side we think the following directions are worth exploring.

First, consider computing minimum-size ensembles:
In \cref{thm:witness-tree-algo}, we showed that \pMTESlong\ and \pMMTESlong\ are both fixed-parameter tractable when parameterized by $\delta$ (the maximum number of dimensions in which a pair of differently labeled examples differ), $D$ (the domain size), $S$ and $s$ (the total tree ensemble size and the maximum size of a tree in the ensemble, respectively), and $\ell$ (the number of trees in the ensemble), respectively.
It would be interesting to investigate the problem for strictly smaller parameterizations.
Of course, lower bounds for \pDTSlong\ also apply to \pMTES\ and \pMMTES.
Hence, these two problems are $W[2]$-hard with respect to $(D,S,\ell)$ and $(D,s,\ell)$, respectively, and NP-hard for constant values of $(\delta,D,\ell)$.
This leaves the parameterized complexity of \pMTES\ for $(\delta,S)$ and of \pMMTES\ for~$(\delta,s,\ell)$ as open problems.
Perhaps one can use the techniques of \cite{EibenOrdyniakPaesaniSzeider23} who showed that related parameterizations lead to tractability for computing minimum-size decision trees. 

Second, we have mostly focused on the complexity with respect to standard parameters of the problems.
It would be interesting to explore more fine-grained structural parameters that capture further properties of the input data, such as the distribution of examples in space or how well examples are separated into classes.
Such properties may lead to tractability if they are not present in the hardness constructions that we have given.
\cite{DabrowskiEOPS24} considered a structural parameter called rankwidth which is motivated from a theoretical point of view but it would be interesting to use or formulate parameters that stem from practical data.

Finally, an important ingredient to lots of practically relevant exact algorithms is data reduction.
Parameterized algorithms allow to capture efficient data reduction in terms of so-called polynomial-size problem kernels.
A polynomial-size problem kernel for a specific parameter, such as the desired solution size, is a polynomial-time algorithm that takes an instance for the problem at hand as input and shrinks this instance to a size that is bounded by a polynomial in the parameter.
(For details on kernelization, we refer to textbooks~\cite{kernelization_book,CyFoKoLoMaPiPiSa2015}.)
\cite{GZ24} began exploring the existence of polynomial problem kernels for computing minimum-size decision trees with mostly negative results for standard parameters.
Here it is again interesting to consider structural parameters instead but even preprocessing algorithms without an a priori theoretical guarantee could be of practical interest.

%\bibliography{refs}
%\bibliographystyle{plainurl}

\end{document}